\documentclass{article}
\usepackage[T1]{fontenc}

\usepackage{microtype}
\usepackage{graphicx}
\usepackage{subcaption}
\usepackage{booktabs} 
\usepackage{enumitem}

\usepackage{hyperref}

\usepackage[accepted]{icml2026}
\usepackage{xcolor} 
\renewcommand{\algorithmiccomment}[1]{\hfill \textcolor{gray}{// #1}}

\usepackage{amsmath}
\usepackage{amssymb}
\usepackage{mathtools}
\usepackage{amsthm}
\usepackage{natbib}
\usepackage{aliascnt}
\usepackage[capitalize,noabbrev]{cleveref}
\usepackage{float}
\theoremstyle{plain}
\newtheorem{theorem}{Theorem}[section]
\newtheorem{proposition}[theorem]{Proposition}
\newtheorem{lemma}[theorem]{Lemma}

\theoremstyle{definition}
\newaliascnt{definition}{theorem}
\newtheorem{definition}[definition]{Definition}
\aliascntresetthe{definition}

\newaliascnt{assumption}{theorem}

\aliascntresetthe{assumption}

\theoremstyle{remark}
\newaliascnt{remark}{theorem}

\aliascntresetthe{remark}

\crefname{definition}{Definition}{Definitions}
\Crefname{definition}{Definition}{Definitions}
\crefname{assumption}{Assumption}{Assumptions}
\Crefname{assumption}{Assumption}{Assumptions}
\crefname{remark}{Remark}{Remarks}
\Crefname{remark}{Remark}{Remarks}

\usepackage[textsize=tiny]{todonotes}

\icmltitlerunning{Quantifying and Optimizing Simplicity via Polynomial Representations}

\usepackage{xspace} 
\newcommand{\shortterm}{ED\xspace}
\newcommand{\longterm}{effective degree\xspace}
\newcommand{\longtermnocaps}{effective degree\xspace}

\begin{document}

\twocolumn[
    \icmltitle{Quantifying and Optimizing Simplicity via Polynomial Representations}



  \icmlsetsymbol{equal}{*}

  \begin{icmlauthorlist}
    \icmlauthor{Tianren Zhang}{equal,thu}
    \icmlauthor{Xiangxin Li}{equal,thu}
    \icmlauthor{Minghao Xiao}{equal,thu}
    \icmlauthor{Guanyu Chen}{thu}
    \icmlauthor{Feng Chen}{thu}
  \end{icmlauthorlist}

  \icmlaffiliation{thu}{Department of Automation, Tsinghua University, Beijing, China}

  \icmlcorrespondingauthor{Feng Chen}{chenfeng@mail.tsinghua.edu.cn}

  \icmlkeywords{Simplicity Metric,Neural Networks,Polynomial Representations}

  \vskip 0.3in
]



\printAffiliationsAndNotice{\icmlEqualContribution}

\begin{abstract}
  Deep networks often exhibit a preference for ``simple'' solutions, and such a simplicity bias is widely believed to play a key role in generalization. Yet a broadly applicable, quantitative measure of simplicity remains elusive. We introduce \emph{polynomial representations} as a distribution-aware, low-dimensional surrogate for neural functions: we approximate a network’s predictive behavior along data-dependent interpolation paths using orthogonal polynomial bases, yielding a compact functional representation. We show that the \emph{effective degree} of this representation serves as a practical simplicity metric that is predictive of generalization across tasks and architectures, and consistently outperforms existing generalization proxies such as sharpness. Finally, polynomial representations naturally yield a \emph{differentiable} simplicity regularizer, which consistently improves generalization in image and text classification, fine-tuning contrastive vision–language models, and reinforcement learning.\footnote{Code is available at: \url{https://github.com/xinzaixinzai/Effective-Degree}}
\end{abstract}

\section{Introduction}

Deep learning models routinely generalize well despite being heavily over-parameterized. A widely discussed explanation is \emph{simplicity bias}: when many solutions fit the training data, the learning dynamics tend to converge to functions that are “simpler” in an appropriate sense, and such simplicity is conjectured to underpin generalization~\citep{hinton1993keeping,gunasekar_implicit_2018,perez_deep_2019,lyu_gradient_2021,huh_low-rank_2023}. However, turning this intuition into a practical tool faces a fundamental obstacle—what exactly does “simple” mean for a neural function, and how can we quantify it in a way that is useful for modern deep learning?

A useful simplicity measure for neural networks should ideally satisfy three desiderata: \textbf{(i) generality} across tasks and architectures, \textbf{(ii) quantifiability} at scale for trained models, and \textbf{(iii) optimizability} via (approximate) differentiability.
Existing proposals typically satisfy only part of these. Provable implicit-bias characterizations such as max-margin or minimum-norm solutions~\citep{soudry_implicit_2018,gunasekar_characterizing_2018} provide clean notions of simplicity, but largely hold in restricted regimes and do not extend directly to deep nonlinear models~\citep{chizat_lazy_2019,zhang_feature_2024}.
Information-theoretic notions based on compression or description length~\citep{schmidhuber_discovering_1997,dziugaite_computing_2017,arora_stronger_2018,goldblum_position_2024} offer universal principles, but are difficult to quantify for neural functions and rarely yield a practical training objective. Geometry- or capacity-based notions such as splines and linear-region counts~\citep{montufar2014number,raghu2017expressive} are conceptually aligned with expressivity and hence complexity, yet are architecture-dependent and hard to estimate at scale.
Meanwhile, many popular proxies typically live in parameter space (e.g., norms, parameter Jacobian, and sharpness~\citep{neyshabur2015norm,sokolic2017robust,keskar_large-batch_2017,lee2023implicit}) and can be sensitive to \emph{reparameterization} and implementation details~\citep{nagarajan_uniform_2019,andriushchenko2023modern}.
Consequently, it remains challenging to define a simplicity metric that is simultaneously general, quantifiable, and optimizable.

Motivated by this gap, in this work we take a \emph{function-space} approach: if simplicity is ultimately a property of the learned function rather than its parameterization, then quantifying simplicity directly in function space is inherently more compatible with generality and robustness to reparameterization—and, when coupled with a differentiable estimator, can also support optimization. The main challenge is that for an arbitrary high-dimensional function, defining a meaningful notion of simplicity is nontrivial on its own. Our strategy is therefore to introduce a \emph{surrogate function family} that (i) admits a \emph{natural} and tractable notion of simplicity and (ii) can be estimated reliably from data. This leads us to a natural choice of polynomials, which form a simple yet expressive surrogate family—capable of approximating continuous functions on compact domains~\citep{stone1948generalized,trefethen2019approximation}—while admitting a natural notion of complexity via degree. Concretely, we approximate a network’s predictive behavior along data-dependent \emph{interpolation paths} using orthogonal polynomial bases. Paths are formed by interpolating between data points either in the input space (for continuous inputs) or in an appropriate continuous representation space (for discrete inputs), yielding compact, low-dimensional \emph{polynomial representations} that are stable to fit and sidestepping the combinatorial growth of multivariate polynomial bases.

This representation leads to an intuitive notion of simplicity: degree of nonlinearity.
Along an interpolation path, low-degree surrogates are close to affine behavior, while higher degrees are needed when predictions vary more nonlinearly.
Notably, we show that the \emph{effective degree} of the fitted polynomial expansion provides a general and quantifiable proxy for functional simplicity: across architectures and datasets, it tracks generalization throughout training and is more predictive than widely used measures such as sharpness~\citep{keskar_large-batch_2017,kwon_asam_2021}.

Beyond measurement, polynomial representations can also lead to simplicity-aware learning objectives. Thanks to the differentiable polynomial fitting procedure, polynomial representations naturally yield a differentiable and tractable simplicity regularizer that explicitly penalizes high-degree components of the learned function. Empirically, we show that incorporating this regularizer into training consistently improves generalization across diverse settings, including image and text classification, fine-tuning contrastive vision–language models, and reinforcement learning.

\textbf{Contributions.}\vspace{-0.6em}
\begin{enumerate}
\item We propose polynomial representations, a functional surrogate for neural networks obtained by fitting orthogonal polynomials along interpolation paths.
\item We show that the effective degree of this representation provides a simplicity metric that is general and quantifiable across tasks and architectures, tracks generalization, and outperforms existing metrics.
\item We derive a differentiable simplicity regularizer from polynomial representations and demonstrate consistent gains across multiple modalities and tasks.
\end{enumerate}

\section{Related Work}

\textbf{Simplicity bias and implicit bias of neural networks.}
A large body of work aims to explain
generalization in over-parameterized networks via implicit regularization from gradient-based training~\citep{neyshabur2014search,zhang_understanding_2017,belkin_reconciling_2019}. In linear models and certain homogeneous networks, gradient descent provably converges to minimum-norm or max-margin solutions~\citep{soudry_implicit_2018,gunasekar_characterizing_2018,ji_implicit_2019,lyu_gradient_2021,abbe_generalization_2023,zhang_when_2025}, yielding clean simplicity notions in restricted regimes. In deep nonlinear regimes, however, simplicity bias is more elusive: the learned function often depends intricately on architecture, data distribution, and optimization details. 
Empirically, training dynamics often fit coarse structure before memorizing noise~\citep{arpit2017closer,nakkiran_sgd_2019} and exhibit spectral or frequency biases in which low-frequency components are learned earlier or more readily~\citep{rahaman2019spectral,xu_frequency_2020,teney_neural_2024}. Related function-space analyses use Fourier features, kernels, and linearization~\citep{tancik2020fourier,lee_deep_2018,jacot_neural_2018,lee2019wide}, as well as approximation-theoretic function classes such as Barron-type norms~\citep{barron2018approximation}.
Finally, piecewise-linear or spline perspectives exploit the fact that ReLU networks induce
spline-like functions with locally simple structure~\citep{montufar2014number,raghu2017expressive,balestriero2018spline}, motivating geometric interpretations of representation complexity.
While insightful, these lines of work do not directly yield a general, computable simplicity metric.
Our work aligns with this function-space view of simplicity bias, but operationalizes it differently by using polynomial representations as a functional surrogate, which naturally admits a general, computable proxy for functional simplicity.

\textbf{Generalization proxies for neural networks.}
A separate literature proposes post-hoc generalization measures---quantities computed from a trained network that aim to predict or upper bound test error.
These include norm- and margin-based measures~\citep{neyshabur2015norm, bartlett_spectrally-normalized_2017}, sharpness measures based on local curvature or sensitivity~\citep{keskar_large-batch_2017, dinh2017sharp}, as well as PAC-Bayes and compression viewpoints~\citep{mcallester_pac-bayesian_1999,dziugaite_computing_2017,arora_stronger_2018}. However, many of these measures are sensitive to reparameterization or implementation details~\citep{dinh2017sharp}, and their empirical reliability can vary substantially across settings~\citep{jiang_fantastic_2019}. Our metric instead comes from a function-space approximation and is \emph{architecture-agnostic} by construction.
Finally, geometric complexity notions such as region counts have also been explored as generalization proxies in classification~\citep{somepalli2022can,li_understanding_2025}, but are typically intractable to compute at scale and do not directly provide a differentiable training-time signal.

\begin{figure*}[t]
    \centering
    \includegraphics[width=0.95\textwidth]{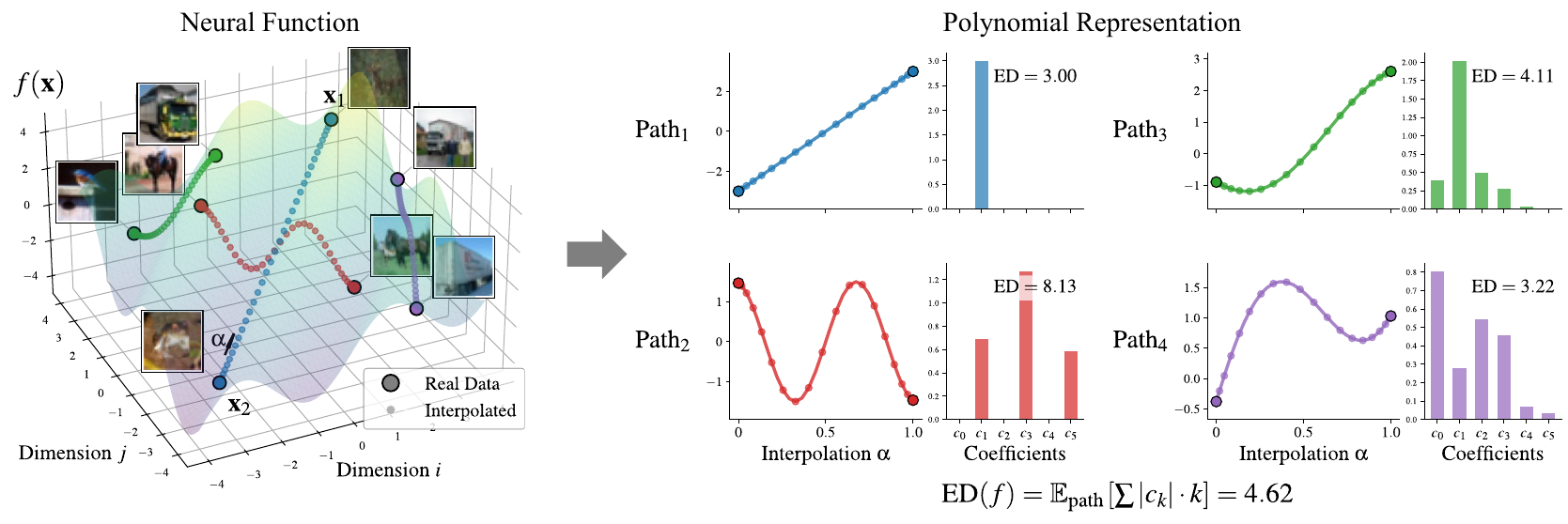} \vspace{-0.7em}
    \caption{\textbf{Method overview.} 
    \textbf{Left:} An illustration of the functional landscape of a neural function $f$ by sampling interpolation paths between data points. 
    \textbf{Right:} The function $f$'s output along these paths is approximated using polynomial expansions. Coefficient histograms reveal the complexity: smoother paths (e.g., Path$_1$) yield low-degree coefficients, whereas oscillating paths (e.g., Path$_2$) show significant high-degree components. 
    The \emph{effective degree (ED)}, computed by averaging coefficient-weighted degrees of fitted polynomials over multiple interpolation paths, serves as a differentiable proxy for functional simplicity.}
    \label{fig:method_overview}
    \vspace{-0.8em}
\end{figure*}

\textbf{Regularization for improving generalization.}
Many methods explicitly regularize training to encourage generalization, including weight decay, margin-based penalties~\citep{liu2016large,elsayed2018large}, Jacobian or smoothness regularization~\citep{drucker1992improving,sokolic2017robust,hoffman2019robust,lee2023implicit}, and sharpness-aware training~\citep{foret_sharpness-aware_2021,kwon_asam_2021}.
In comparison, our regularizer is defined in function space rather than on the model parameters, and is therefore invariant to reparameterization. Empirically, we compare our method with representative regularizers in this category.

\section{Polynomial Representations}
\label{sec:poly_repr}

We estimate functional simplicity by approximating a predictor with a polynomial surrogate and extracting a degree-based notion of complexity.
We start from a general polynomial-basis approximation for multivariate inputs, discuss its intractability in high dimensions, and then present key ingredients of our approach: (i) restrict the network to data-dependent one-dimensional interpolation paths, (ii) reduce high-dimensional outputs by PCA, and (iii) ensure numerical stability via orthogonal polynomial bases and a randomized cosine sampling strategy.

\textbf{Polynomial basis approximation.}
Let $f:\mathbb{R}^d\to\mathbb{R}$ be a (possibly unknown) function.
Given a family of polynomial basis functions $\mathcal{B}=\{\phi_{\boldsymbol{\alpha}}(\mathbf{x})\}_{\boldsymbol{\alpha}\in\mathcal{I}}$ indexed by multi-indices $\boldsymbol{\alpha}$ (e.g., monomials), we consider the truncated expansion
\begin{equation}
    f(\mathbf{x}) \approx \sum\nolimits_{\boldsymbol{\alpha}\in\mathcal{I}_K} c_{\boldsymbol{\alpha}}\,\phi_{\boldsymbol{\alpha}}(\mathbf{x}),
    \label{eq:multivar_poly}
\end{equation}
where $\mathcal{I}_K$ collects basis elements up to a prescribed maximum degree $K$.
The coefficients $\{c_{\boldsymbol{\alpha}}\}$ can be obtained by least-squares fitting from point evaluations of $f$.

However, directly fitting \cref{eq:multivar_poly} in the ambient input space is prohibitive when $d$ is large: the number of basis functions in $\mathcal{I}_K$ grows combinatorially with $d$ and $K$, making both fitting and coefficient-based complexity estimation intractable.
Moreover, modern predictors may also have high-dimensional outputs (e.g., many logits or token-level outputs), which further increases the computational burden if one were to fit a separate surrogate per coordinate.

\textbf{Input reduction.}
To obtain a tractable estimate of functional simplicity, we restrict $f$ to data-dependent one-dimensional paths.
Concretely, given $\mathbf{x}_1,\mathbf{x}_2$ drawn from a data distribution $\mathcal{D}$ on $\mathbb{R}^d$, we define the interpolation path
\begin{equation}
    \mathbf{x}(\alpha)=\alpha\mathbf{x}_1+(1-\alpha)\mathbf{x}_2,\quad \alpha\in[0,1],
    \label{eq:path}
\end{equation}
and the corresponding univariate restriction
\begin{equation}
    g_{\mathbf{x}_1,\mathbf{x}_2}(\alpha)\coloneqq f(\mathbf{x}(\alpha)).
\end{equation}
We then fit a \emph{univariate} polynomial surrogate $P_{\mathbf{x}_1,\mathbf{x}_2}(\alpha)$ to $g_{\mathbf{x}_1,\mathbf{x}_2}(\alpha)$ and read off complexity from its coefficients.
Averaging over many random pairs yields an estimate anchored to the data distribution while avoiding multivariate polynomial growth.
Similar approaches of constraining high-dimensional input space to low- or one-dimensional subspace have also been used by prior work~\citep{fridovich-keil_spectral_2022,teney_we_2025,li_understanding_2025}.
However, a natural question is whether restricting to interpolation paths discards the very notion of polynomial complexity/simplicity we aim to measure.
In particular, \emph{algebraic degree} is a natural complexity proxy for polynomials, but after reducing the input to one-dimensional interpolations, the resulting univariate restriction could in principle have a smaller degree due to cancellations.
Fortunately, we can show that for multivariate polynomials, such degree drops occur only on a measure-zero set of interpolation directions; hence averaging over random interpolation paths preserves the degree ordering almost surely.

\begin{theorem}[Order preservation of degree via interpolation paths]
\label{thm:path_preserve_complexity}
Let $P_1,P_2:\mathbb{R}^d\to\mathbb{R}$ be nonzero polynomials with degrees $D_i=\deg(P_i)$.
Let $\mathbf{x}_1,\mathbf{x}_2\stackrel{\mathrm{i.i.d.}}{\sim}\mathcal{D}$ and assume that there exists a nonempty open set $U\subset\mathbb{R}^d$ such that $P(\mathbf{x}\in U)=1$ and $\mathcal{D}$ admits a density on $U$.
Define
\[
d_{P_i}(\mathbf{x}_1,\mathbf{x}_2)=
\deg_{\alpha} P_i\bigl(\alpha\mathbf{x}_1+(1-\alpha)\mathbf{x}_2\bigr).
\]
Then for i.i.d.\ pairs $(\mathbf{x}_1^{(j)},\mathbf{x}_2^{(j)})_{j=1}^n$,
the empirical averages
\[
\widehat d_n(P_i)=\frac{1}{n}\sum\nolimits_{j=1}^n d_{P_i}\left(\mathbf{x}_1^{(j)},\mathbf{x}_2^{(j)}\right)
\]
satisfy: if $D_1>D_2$, then $\widehat d_n(P_1)>\widehat d_n(P_2)$ for all sufficiently large $n$ almost surely.
\end{theorem}

\begin{proof}[Proof sketch]
This result follows from Lemma~\ref{lem:degree_preservation}, which states that
under these assumptions, the interpolation direction $\mathbf{x}_1-\mathbf{x}_2$ avoids the Lebesgue-null zero set of the leading homogeneous part $P_D$ almost surely, and the univariate restriction $\alpha\mapsto P(\alpha\mathbf{x}_1+(1-\alpha)\mathbf{x}_2)$ preserves degree.
The key ingredient is the classical fact that the zero set of a nonzero polynomial has Lebesgue measure zero~\citep{basu2006algorithms,mityagin2015zero}.
See Appendix~\ref{appsubsec:proof_theorem_complexity} for the full proof.
\end{proof}

\textbf{Effective degree.}
While Theorem~\ref{thm:path_preserve_complexity} justifies path-based dimensionality reduction from the perspective of algebraic degree, directly using it as a practical complexity metric is often brittle:
it is sensitive to small coefficient perturbations and may be dominated by negligible high-order components.
Motivated by this, we define our simplicity metric with a coefficient-weighted surrogate, which we call the \emph{effective degree}.
Concretely, let a fitted univariate polynomial surrogate be represented in its chosen polynomial basis as
\begin{equation}
    P(\alpha)=\sum\nolimits_{k=0}^{K} c_k\phi_k(\alpha),
    \label{eq:univar_basis}
\end{equation}
where $\{\phi_k\}_{k=0}^K$ is the basis.
We define the effective degree in terms of the fitted coefficients $\{c_k\}$.

\begin{definition}[Effective degree]
\label{def:effective_degree}
Given \cref{eq:univar_basis}, the \emph{\longtermnocaps (\shortterm)} of $P$ is
\begin{equation}
    \mathrm{\shortterm}(P) \coloneqq \sum\nolimits_{k=0}^{K} |c_k|\cdot k,
\end{equation}
and the normalized version is
\begin{equation}
    \mathrm{\shortterm}_{\mathrm{norm}}(P)
    \coloneqq
    \frac{\sum_{k=0}^{K} |c_k|\,k}{\sum_{k=0}^{K} |c_k|}.
\end{equation}
\end{definition}

For vector-valued surrogates $\mathbf{P}=(P_1,\dots,P_m)$, we set
$\mathrm{\shortterm}(\mathbf{P})=\frac{1}{m}\sum_{j=1}^m \mathrm{\shortterm}(P_j)$ (and similarly for $\mathrm{\shortterm}_{\mathrm{norm}}$).
Unlike algebraic degree, \shortterm is Lipschitz in the fitted coefficients, hence more robust to fitting noise. Note that the specific choice of absolute-value weighting here is not unique; we also tested alternative variants such as quadratic weighting and found that the current definition performs best empirically. We posit that this is due to the absolute-value weighting naturally introducing a desirable scale-invariance of the gradient with respect to coefficient magnitude.

\textbf{Output reduction.}
When the output of $f$ is high-dimensional, fitting a separate surrogate for each coordinate can be costly and statistically noisy.
We therefore employ a generic, optional output-compression step before polynomial fitting: we learn a low-dimensional linear subspace that captures most output variation on the sampled path points, and fit polynomials only in this subspace.
In particular, we apply PCA \citep{hotelling1933analysis} to the outputs sampled along each specific interpolation path, retain the top $m$ components, and fit polynomials to the resulting projected outputs, resulting in a \emph{per-path} local dimensionality reduction.

\textbf{Numerical stability.}
To make coefficient estimation numerically stable, we choose an \emph{orthogonal polynomial basis}. Orthogonality improves conditioning and reduces coefficient instability~\citep{trefethen2019approximation}.
In particular, we use Chebyshev polynomials, defined recursively by
$T_0(x)=1,\quad T_1(x)=x,\quad T_{k+1}(x)=2xT_k(x)-T_{k-1}(x),$
and we fit expansions of the form $P(\alpha)\approx \sum_{k=0}^{K} c_k T_k(2\alpha-1)$.
Chebyshev fitting is typically paired with sampling at (shifted) Chebyshev nodes on $[0,1]$:
\begin{equation}
    \alpha_i=\frac{1}{2}\left(1-\cos\frac{(2i-1)\pi}{2r}\right),\quad i=1,\dots,r,
    \label{eq:cheb_nodes}
\end{equation}
which clusters points near the endpoints and mitigates Runge-type instabilities.
In addition, we propose a randomized variant that retains this geometry while providing stratified randomness, which is helpful when $r$ is small and we average over many paths.

\begin{definition}[Randomized cosine sampling]
\label{def:cos_sample}
Let $r$ be the sampling resolution.
We draw
\begin{equation}
\theta_i \sim U\!\left[\frac{(i-1)\pi}{r}, \frac{i\pi}{r}\right], \quad i=1,\dots,r,
\label{eq:cos_sample}
\end{equation}
and map to $[0,1]$ by
$
\alpha_i = \tfrac{1}{2}\bigl(1 - \cos \theta_i\bigr).
$
\end{definition}
Randomized cosine sampling can be viewed as a stratified version of sampling according to the Chebyshev measure, which is known to improve the stability of polynomial least-squares fitting~\citep{cohen2013stability}.

\textbf{Estimator.}
Let $P_{\mathbf{x}_1,\mathbf{x}_2}$ denote the fitted degree-$K$ univariate polynomial surrogate to
$g_{\mathbf{x}_1,\mathbf{x}_2}(\alpha)=f(\mathbf{x}(\alpha))$ using a sampling scheme and the chosen orthogonal basis.
We estimate network simplicity by averaging effective degrees over random interpolation paths:
\begin{equation}
\widehat{\mathrm{\shortterm}}(f)
=
\mathbb{E}_{\mathbf{x}_1,\mathbf{x}_2 \sim \mathcal{D}}
\left[
\mathrm{\shortterm}\!\left(P_{\mathbf{x}_1,\mathbf{x}_2}\right)
\right].
\label{eq:ed_estimator}
\end{equation}
In practice, we approximate these expectations with finite samples of pairs and path points.


\section{Effective Degree Tracks Generalization}
\label{sec:generalization_prediction}

\begin{figure*}[t]
	\centering
	\includegraphics[width=\textwidth]{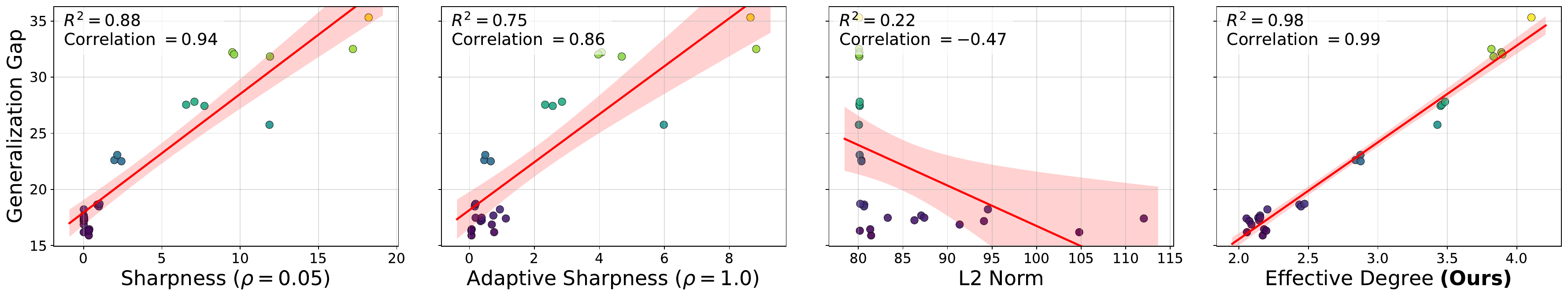}\vspace{-0.2em}
	\caption{Correlations between \longterm, sharpness-based measures, and parameter $L_2$ norm with generalization gap for ResNet18 on CIFAR-10. Effective degree exhibits the strongest linear correlation.
	Points with lighter colors represent models with larger generalization gaps (same for other figures).
		Solid red lines indicate least-squares linear fits with 95\% confidence intervals.
	}\vspace{-0.4em}
	\label{fig:resnet_all}
\end{figure*}

\begin{figure*}[t]
	\centering
	\includegraphics[width=\textwidth]{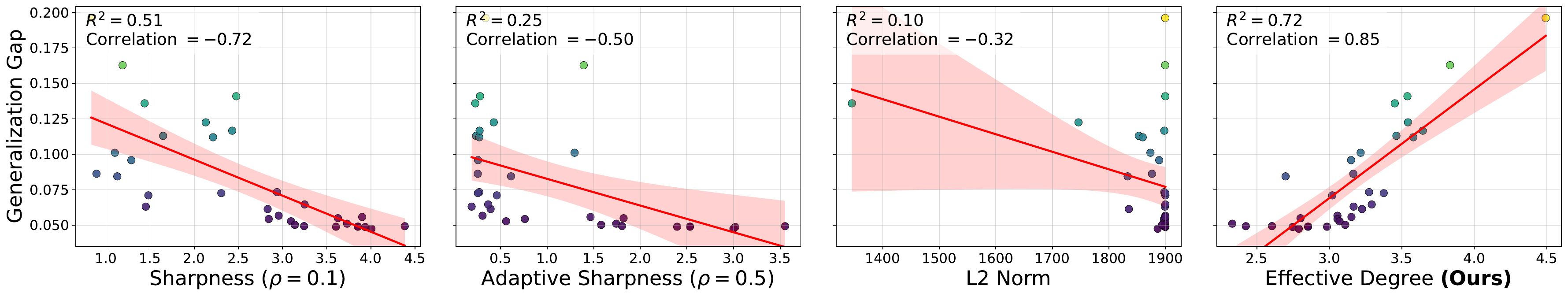}\vspace{-0.2em}
	\caption{
		Correlations between \longterm, sharpness-based measures, and parameter $L_2$ norm with generalization gap for CLIP ViT-B/32 fine-tuned on ImageNet.
		Effective degree exhibits a positive correlation with the generalization gap, whereas all other measures correlate negatively. Solid red lines indicate least-squares linear fits with 95\% confidence intervals.\vspace{-1.0em}
	}
	\label{fig:imagenet_clip_mixup_active}
\end{figure*}

In this section, we show that \longterm (\shortterm) tracks generalization performance and compares favorably with existing generalization proxies. Specifically, we observe a strong correlation between \shortterm and generalization gap across image classification benchmarks, and find that it accurately captures phase transitions in grokking. Finally, we discuss the relation between \shortterm and standard generalization bounds.

\subsection{\shortterm Correlates with Generalization Gap}
\label{subsec:correlation_gap}

We study Pearson correlation between \shortterm and generalization gap (train minus test accuracy) across different training configurations for ResNets~\citep{he_deep_2016} / ViTs~\citep{dosovitskiy_image_2021} on CIFAR-10~\citep{krizhevsky_learning_2009} and fine-tuned CLIP~\citep{radford_learning_2021} on ImageNet~\citep{deng_imagenet:_2009}.
We compare \shortterm to sharpness~\citep{foret_sharpness-aware_2021}, adaptive sharpness~\citep{kwon_asam_2021}, and parameter $L_2$ norm, reporting the best correlation over metric variants (raw vs.\ normalized \shortterm; radius sweeps for sharpness). See \cref{app:exp_correlation} for details.



\textbf{CIFAR-10.}
We use ResNet18 and ViT-Tiny as representative neural network architectures.
For each architecture, we train models across 27 distinct hyperparameter configurations and report results averaged over three random seeds for each configuration.
As illustrated in \cref{fig:resnet_all}, on ResNet18, \shortterm shows the strongest correlation with the gap and consistently outperforms both sharpness variants, while the parameter $L_2$ norm correlates negatively.
ViT-Tiny results show similar trends and are detailed in \cref{app:vit_corr}.


\textbf{ImageNet.}
We analyze fine-tuned CLIP ViT-B/32 models from~\citet{wortsman_model_2022} under multiple recipes.
Because the fine-tuning recipe such as mixup~\citep{zhang_mixup_2018} systematically shifts both the generalization gap distribution and the scale of metrics, we report correlations in a recipe-stratified manner to avoid confounding.
Under mixup (\cref{fig:imagenet_clip_mixup_active}), \shortterm correlates positively with the gap, whereas sharpness-based measures correlate \emph{negatively}, consistent with prior work~\citep{andriushchenko2023modern,silva_hide_2025}; the parameter norm is only weakly correlated.
Non-mixup recipes show the same qualitative relationship between \shortterm and generalization gap; see \cref{app:imagenet_corr}.

\subsection{\shortterm Tracks Grokking Transitions}
\label{subsec:grokking}

We also study grokking~\citep{power_grokking_2022} on Modular Division over $\mathbb{Z}_{97}$ with a 30\% training split, where models first memorize and only later transition to a generalizing solution. We investigate whether \shortterm can capture this delayed transition compared to established generalization proxies. See \cref{app:exp_grokking} for detailed experimental protocols.


\textbf{\shortterm aligns with the grokking transition.}
We track \shortterm alongside parameter $L_2$ norm, sharpness, and adaptive sharpness (best over $\rho$, shown with $\rho=0.05$).
As shown in~\cref{fig:grokking}, \shortterm rises during memorization, peaks near the drop in validation loss, and then decreases, suggesting the eventual generalizing solution is simpler under our metric.
In contrast, the parameter norm increases monotonically and sharpness measures either decay early or fluctuate, providing a less clear transition signal.


\begin{figure}[t]
	\centering
	\includegraphics[width=\linewidth]{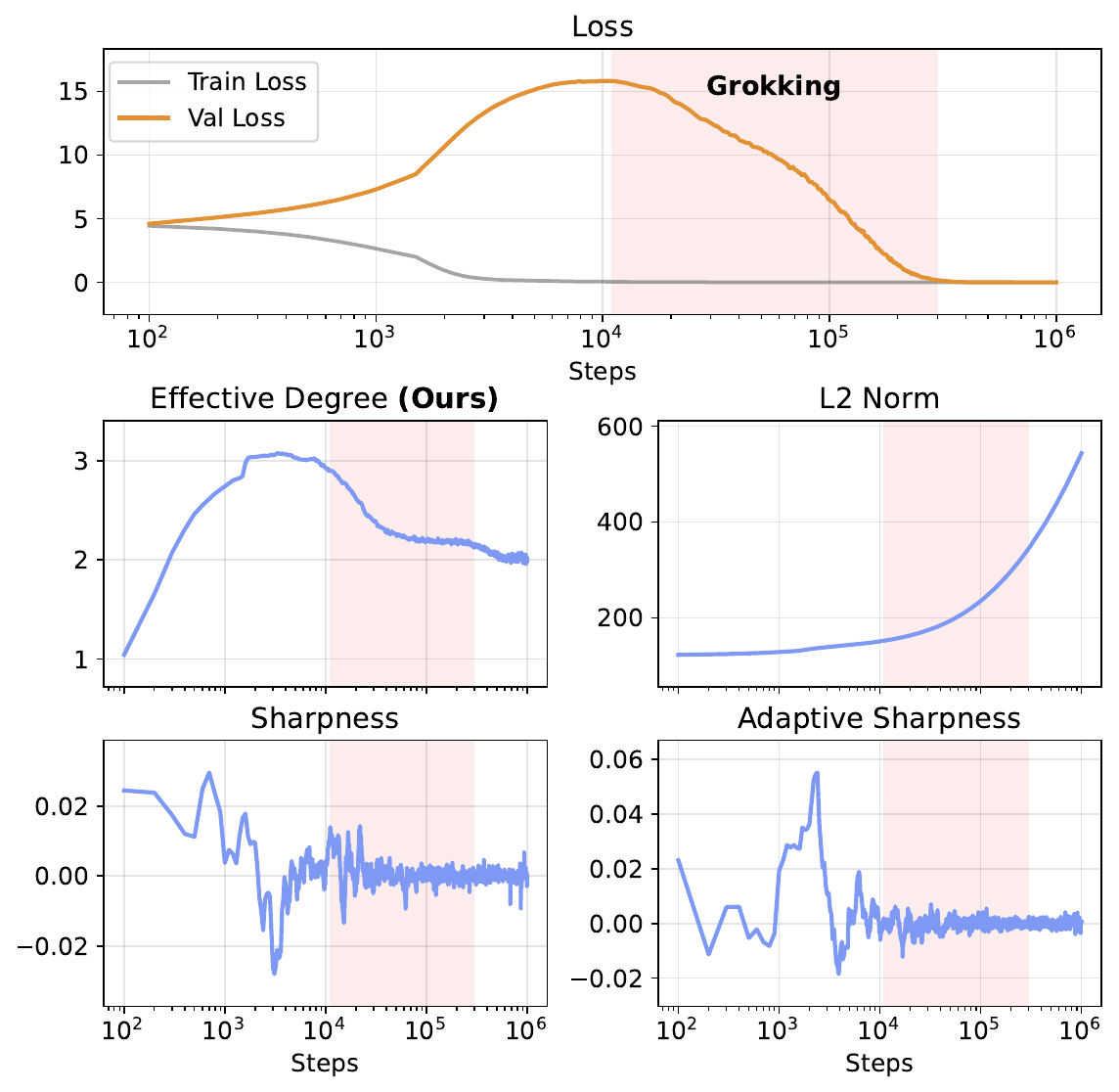}\vspace{-0.5em}
	\caption{
Tracking grokking dynamics.
\textbf{Top panel:} validation loss.
\textbf{Bottom four panels:} \longterm versus baselines; only \longterm peaks at the transition and decreases thereafter.
}
	\label{fig:grokking}
\end{figure}

\subsection{Connecting \shortterm to Generalization Bounds}
\label{subsec:connetcing_ed_gen_bounds}

A classical sanity check for degree-based simplicity is that low-degree polynomials form a low-capacity hypothesis class in standard learning theory.
For multivariate polynomials of total degree at most $K$ in $d$ variables, the number of monomials is
$M=\binom{d+K}{K}$, and the resulting class can be viewed as a linear predictor in an $M$-dimensional monomial feature space, yielding standard uniform-convergence generalization guarantees whose dependence worsens as $K$ grows; see, e.g., \citep{shalev-shwartz_understanding_2014}.
Notably, ED resembles a weighted $\ell_1$ constraint on polynomial coefficients; for linear predictors with bounded $\ell_1$ norm and bounded features, classical Rademacher complexity bounds scale only logarithmically with the feature dimension, suggesting that concentrating coefficient mass on lower degrees, as encouraged by smaller ED, corresponds to tighter capacity control~\citep{bartlett_rademacher_2002}. However, classical uniform-convergence and Rademacher bounds have been known to be \emph{vacuous} for deep neural networks~\citep{dziugaite_computing_2017,nagarajan_uniform_2019}, for which we show that \shortterm still serves as an effective generalization proxy.

For special model classes such as polynomial neural networks (PNNs), the connection between ED and existing proxies for simplicity or generalization could be made clearer. In particular, we include a controlled study on PNNs in \cref{app:ablation_pnn}, where \longterm preserves the ground-truth algebraic-degree ordering of polynomial target functions, which remains stable under basis changes and PCA-based output reduction.

\section{Effective Degree Regularization}
\label{sec:ed_regularization}

In this section, we show that \longterm (\shortterm) admits analytic gradients through the polynomial-fitting procedure, enabling end-to-end optimization of functional simplicity as an explicit regularizer.

\subsection{Differentiability of \shortterm and Stable Implementation}
\label{subsec:diff_ed}

To use \shortterm as a training-time regularizer, we need to differentiate it with respect to network outputs (and hence parameters) through the polynomial fitting step.
Concretely, along each sampled interpolation path we fit Chebyshev coefficients by least squares and compute \shortterm as a weighted $\ell_1$ function of these coefficients.
Both steps admit closed-form gradients.
The following proposition provides the resulting gradient with respect to the sampled outputs.

\begin{proposition}[Differentiability of \longterm]
\label{prop:diff_ed}
Let $\mathbf{y}\in\mathbb{R}^{r}$ denote model outputs sampled at $r$ nodes $\{\alpha_i\}_{i=1}^r$ along an interpolation path, and let $\mathbf{c}=[c_0,\dots,c_{K}]^\top$ be the coefficients of the degree-$K$ Chebyshev least-squares fit.
Define $\mathbf{d}=[0,\dots,K]^\top$ and the design matrix $\mathbf{T}\in\mathbb{R}^{r\times(K+1)}$ by
\[
\mathbf{T}_{i,k} \;=\; T_k(2\alpha_i-1),\ i=1,\dots,r,\; k=0,\dots,K,
\]
where $T_k$ is the $k$-th Chebyshev basis function.
Then, whenever $\mathbf{T}^\top\mathbf{T}$ is invertible and $c_k\neq 0$ for all $k$,
\begin{equation}
\frac{\partial\,\mathrm{\shortterm}}{\partial \mathbf{y}}
\;=\;
\mathbf{T}(\mathbf{T}^\top\mathbf{T})^{-1}\bigl(\mathrm{sign}(\mathbf{c})\odot \mathbf{d}\bigr),
\end{equation}
where $\odot$ is the element-wise product.
\end{proposition}

The proof is deferred to Appendix~\ref{appsubsec:proof_prop_diff}.


While Proposition~\ref{prop:diff_ed} provides the analytical gradient, direct computation involves the inversion of the matrix $\mathbf{T}^\top\mathbf{T}$, which can be numerically unstable when $\mathbf{T}$ is ill-conditioned under stochastic sampling.
We therefore solve the damped normal equations to obtain the regularized coefficients $\mathbf{c}_\epsilon$:
\begin{equation}
\label{eq:regularized_system}
(\mathbf{T}^\top\mathbf{T} + \epsilon \mathbf{I}) \mathbf{c}_\epsilon = \mathbf{T}^\top \mathbf{y},
\end{equation}
with a small damping factor $\epsilon>0$.
We can then compute $\mathbf{c}_\epsilon$ via a linear solver,
\begin{equation}
\mathbf{c}_\epsilon = \texttt{LinearSolve}(\mathbf{T}^\top\mathbf{T} + \epsilon \mathbf{I},\, \mathbf{T}^\top \mathbf{y}),
\end{equation}
which remains differentiable in modern autodiff frameworks (e.g., PyTorch~\citep{paszke_pytorch_2019}) and avoids explicit matrix inversion. In practice, we use LU-based linear solves provided by PyTorch. The overhead is modest because each fit is only a small per-path linear solve.
The explicit gradient derivation for this damped implementation is provided in Appendix~\ref{appsubsec:proof_prop_diff}.

\subsection{\shortterm Regularization Objective}
\label{subsec:ed_objective}

We add \shortterm as a minibatch complexity penalty.
Given a minibatch $\mathcal{B}=\{(x_b,t_b)\}_{b=1}^B$, we sample $n_p$ input pairs $\{(x_1^{(i)}, x_2^{(i)})\}_{i=1}^{n_p}$ from $\mathcal{B}$ and evaluate the model along $r$ nodes $\{\alpha_\ell\}_{\ell=1}^r$ either from Chebyshev nodes or randomized cosine sampling.
For each pair we fit a degree-$K$ Chebyshev polynomial to the resulting 1D path and compute its \shortterm. When the output dimension is large, we first project outputs to a low-dimensional subspace via PCA before fitting, as described in~\cref{sec:poly_repr}. Note that in our implementation, PCA is performed separately for each sampled path. For the outputs collected along a given path, we compute a path-specific PCA projection, project the path samples into the corresponding low-dimensional subspace, and then fit the polynomial surrogate in this reduced space. Gradients are computed starting from the PCA outputs and differentiated through the PCA decomposition process: after obtaining the path-specific projected representation, the ED regularization loss is computed and differentiated based on those projected coordinates. 
Averaging over sampled pairs then yields the minibatch estimator $\widehat{\mathrm{\shortterm}}_{\mathcal{B}}$.
We then optimize
\begin{equation}
\mathcal{L}(\theta;\mathcal{B})
=
\mathcal{L}_{\text{task}}(\theta;\mathcal{B})
+
\lambda\,\widehat{\mathrm{\shortterm}}_{\mathcal{B}},
\end{equation}
where $\lambda>0$ controls the regularization strength.
The full procedure is summarized in \cref{alg:shortterm_regularization} in~\cref{appsec:algo}.

\paragraph{Label-anchored \shortterm for classification.}
\label{sec:label_anchored}
For classification problems, cross-entropy encourages rapid deviation from near-constant predictions early in training, which could lead to an optimization conflict with ED.
We mitigate this issue with \emph{label-anchored \shortterm}: we replace the model predictions at the two boundary nodes with the corresponding ground-truth labels when fitting the polynomial surrogate.
For randomized cosine sampling, we fix the boundary angles to $\theta_1=0$ and $\theta_r=\pi$, and sample only the remaining $r-2$ interior nodes as in \cref{def:cos_sample}.
Note that this anchoring strategy is compatible with the output reduction described earlier; when both are employed, we apply the label replacement to the boundary nodes \emph{before} computing the PCA projection.
Unless otherwise mentioned, in supervised classification tasks we will use this variant and still refer to it as \shortterm.

\begin{table*}[t]
	\centering
	\small
    \setlength{\tabcolsep}{4pt}
	\caption{Top-1 test accuracy (\%, mean $\pm$ std over 3 seeds) on CIFAR-10. For SAM/ASAM, we report the best result over the tuned $\rho$ grids. For Jacobian regularization, we report the best result over the tuned $\lambda_{\text{JR}}$ grids (see Appendix~\ref{sec:cifar10_settings}).}
	\label{tab:cifar10_top1}
	\begin{tabular}{lccccccc}
		\toprule
		& Baseline & Mixup & SAM & ASAM & Jacobian reg. & \shortterm w/o LA (Ours) & \shortterm (Ours) \\
		\midrule
		ViT-Tiny (Aug)
		& $87.80 \pm 1.17$
		& $88.83 \pm 1.48$
		& $87.85 \pm 1.27$
		& $87.85 \pm 1.24$
		& $87.81 \pm 0.17$
		& $90.00 \pm 0.60$
		& $\mathbf{90.82 \pm 0.11}$ \\
		\bottomrule
	\end{tabular}
\end{table*}

\begin{table*}[t]
	\centering
	\setlength{\tabcolsep}{3pt}
	\small
	\caption{In-distribution (ID) and out-of-distribution (OOD) accuracy (\%, mean $\pm$ std over 3 seeds) for CLIP ViT-B/16 and ViT-B/32.}
	\label{tab:imagenet_results}
	\begin{tabular}{l|c|ccccc|c}
		\toprule
		{Method} & {ImageNet (ID)} & {ImageNetV2} & {ImageNet-R} & {ImageNet-A} & {ImageNet Sketch} & {ObjectNet} & {Avg. OOD} \\
		\midrule
		CLIP ViT-B/32 & $ 76.20 \pm 0.02 $ & $64.21 \pm 0.11$ & $56.82 \pm 0.31$ & $20.48 \pm 0.17$ & $39.08 \pm 0.07$ & $39.62 \pm 0.10$ & $44.04 \pm 0.08$  \\
		\quad + \shortterm & $\mathbf{77.14 \pm 0.05}$ & $\mathbf{65.37 \pm 0.18}$ & $\mathbf{58.28 \pm 0.09}$ & $\mathbf{22.03 \pm 0.37}$ & $\mathbf{40.46 \pm 0.24}$ & $\mathbf{40.41 \pm 0.17}$ & $\mathbf{45.31 \pm 0.08}$ \\
		\midrule
		CLIP ViT-B/16 & $81.35 \pm 0.11$ & $70.89 \pm 0.13$ & $65.32 \pm 0.10$ & $36.63 \pm 0.11$ & $45.45 \pm 0.32$ & $50.14 \pm 0.14$ & $53.69 \pm 0.04$ \\
		\quad + \shortterm & $\mathbf{82.19 \pm 0.03}$ & $\mathbf{72.04 \pm 0.27}$ & $\mathbf{66.30 \pm 0.24}$ & $\mathbf{39.81 \pm 0.66}$ & $\mathbf{47.53 \pm 0.17}$ & $\mathbf{50.76 \pm 0.11}$ & $\mathbf{55.29 \pm 0.14}$ \\
		\bottomrule
	\end{tabular}\vspace{-0.7em}
\end{table*}

\section{Empirical Analysis}
\label{sec:eval_regularization}

In this section, we evaluate \shortterm regularization across vision, language, and reinforcement learning.
Concretely,
we train ViTs from scratch on CIFAR-10 and ImageNet (\cref{subsec:image_classification_reg}), fine-tune CLIP on ImageNet (\cref{sec:CLIP_Fine-Tuning}), fine-tune BERT~\citep{devlin_bert_2018} on GLUE tasks~\citep{wang2018glue} (\cref{subsec:glue}), and regularize reinforcement learning agents on Procgen~\citep{cobbe_leveraging_2020} (\cref{subsec:procgen}).
We also include ablation studies on the hyperparameters and several design choices of ED, analyze its failure mode, and provide a computational overhead analysis (\cref{subsec:ablation}).

\subsection{Image Classification on CIFAR-10 and ImageNet}
\label{subsec:image_classification_reg}

On CIFAR-10, we train ViT-Tiny under seven strategies: a standard training baseline, \shortterm regularization with and without label anchoring (LA), mixup~\citep{zhang_mixup_2018}, sharpness-aware minimization (SAM)~\citep{foret_sharpness-aware_2021}, ASAM~\citep{kwon_asam_2021}, and Jacobian regularization~\citep{hoffman2019robust}. We also test the scalability of ED regularization by training ViT-S/16 from scratch on ImageNet, following the original ViT recipe~\citep{dosovitskiy_image_2021} and a stronger recipe from~\citep{beyer2022better} as baselines. See Appendices~\ref{sec:cifar10_settings} and~\ref{sec:imagenet_settings} for details.

\textbf{Results.}
On CIFAR-10, \Cref{tab:cifar10_top1} shows that \shortterm yields the best accuracy, improving over the baseline by +3.02 points, while SAM/ASAM/Jacobian reg. are comparable to the baseline. Notably, \shortterm without LA slightly underperforms the standard \shortterm, yet still outperforms all other methods, suggesting that the primary performance gain stems from explicit \emph{complexity control} rather than \emph{label alignment}. On ImageNet, \Cref{tab:imagenet_scratch} shows consistent gains from \shortterm under both recipes, indicating that \shortterm remains effective on larger-scale training and atop carefully tuned recipes.

\begin{figure}[t]
	\centering
	\begin{subfigure}[b]{0.235\textwidth}
		\centering
		\includegraphics[width=\linewidth]{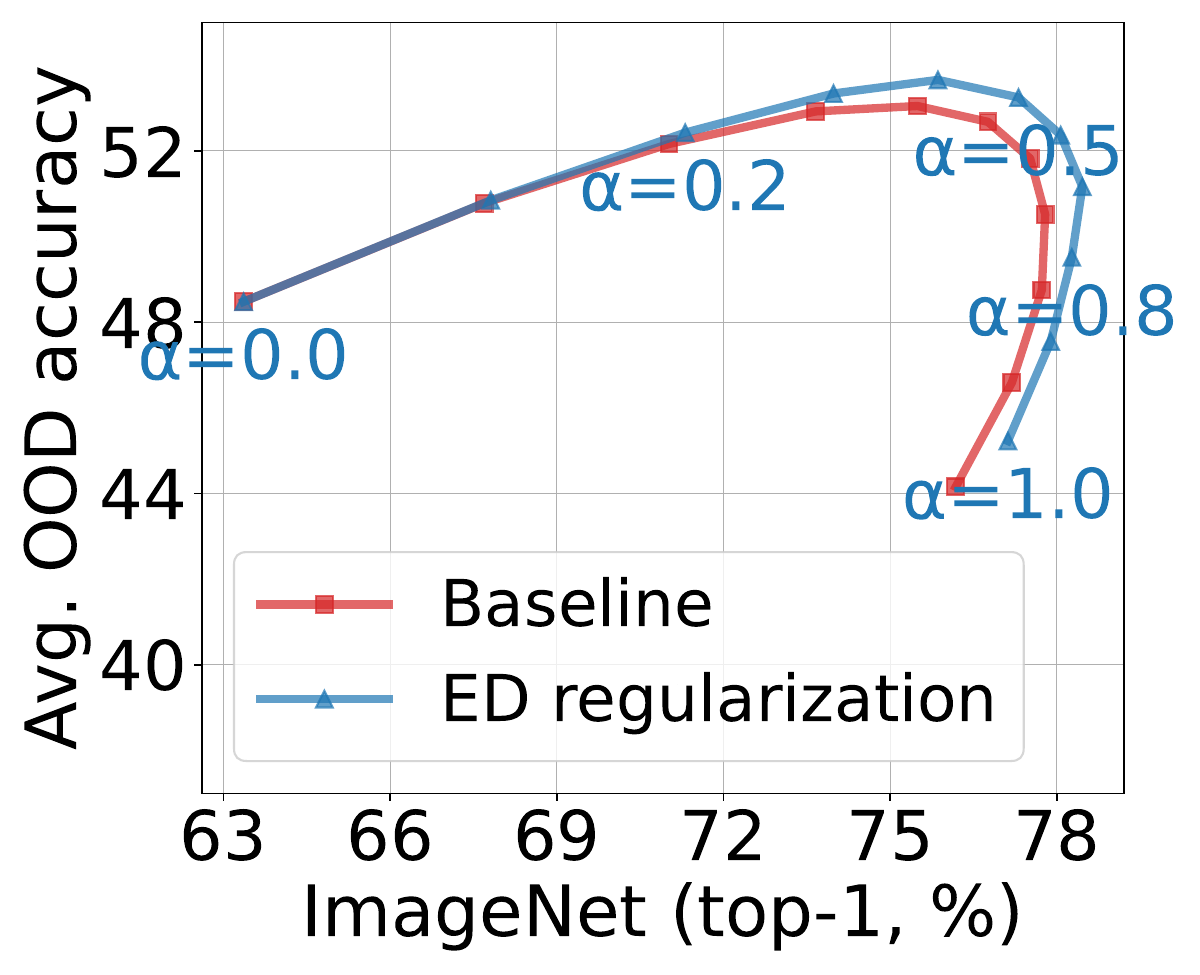}
		\caption{ViT-B/32}
		\label{fig:vitb32_curve}
	\end{subfigure}
	\hfill
	\begin{subfigure}[b]{0.235\textwidth}
		\centering
		\includegraphics[width=\linewidth]{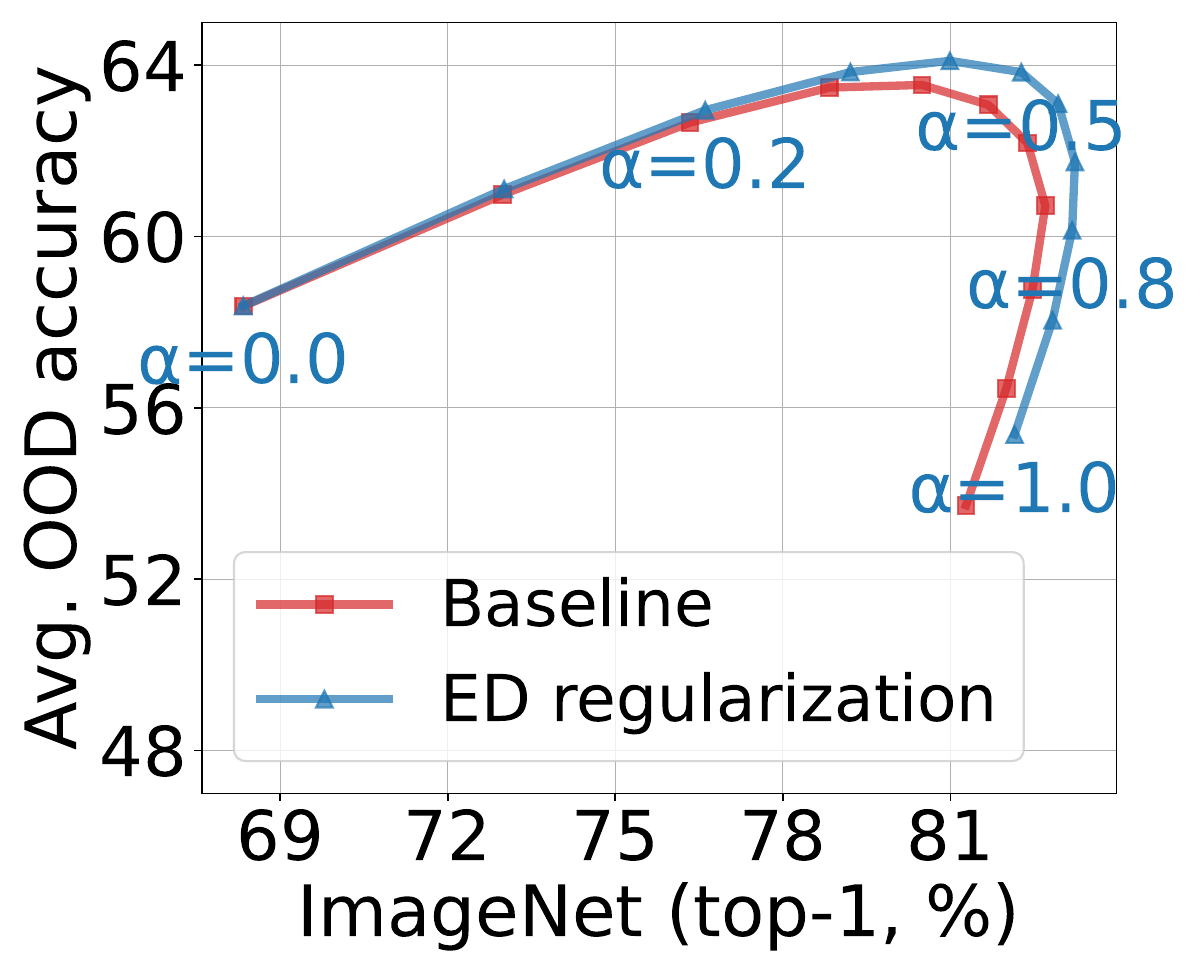}
		\caption{ViT-B/16}
		\label{fig:vitb16_curve}
	\end{subfigure}
	\caption{ImageNet (ID) accuracy vs.\ average OOD accuracy over 5 shifts under weight interpolation ($\alpha\in[0,1]$). \shortterm yields a better trade-off than standard fine-tuning across all $\alpha$.}\vspace{-1.0em}
	\label{fig:robustness_curves}
\end{figure}

\begin{table}[t]
	\centering
	\small
	\caption{Top-1 test accuracy (\%) of ViT-S/16 (mean $\pm$ std over 3 seeds) trained from scratch on ImageNet under two recipes.}
	\label{tab:imagenet_scratch}
	\begin{tabular}{lc}
		\toprule
		{Method} & {Top-1 Accuracy} \\
		\midrule
		ViT-S/16 (original recipe) & $71.37 \pm 0.17$ \\
		\quad + \shortterm & $\mathbf{72.76 \pm 0.16}$ \\
		\midrule
		ViT-S/16 (strong recipe) & $74.42 \pm 0.13$ \\
		\quad + \shortterm & $\mathbf{75.01 \pm 0.11}$ \\
		\bottomrule
	\end{tabular}\vspace{-1.0em}
\end{table}

\subsection{CLIP Fine-Tuning on ImageNet}
\label{sec:CLIP_Fine-Tuning}

\begin{figure*}[t!]
    \centering
    \begin{subfigure}[b]{0.22\textwidth}
        \centering
        \includegraphics[width=\linewidth]{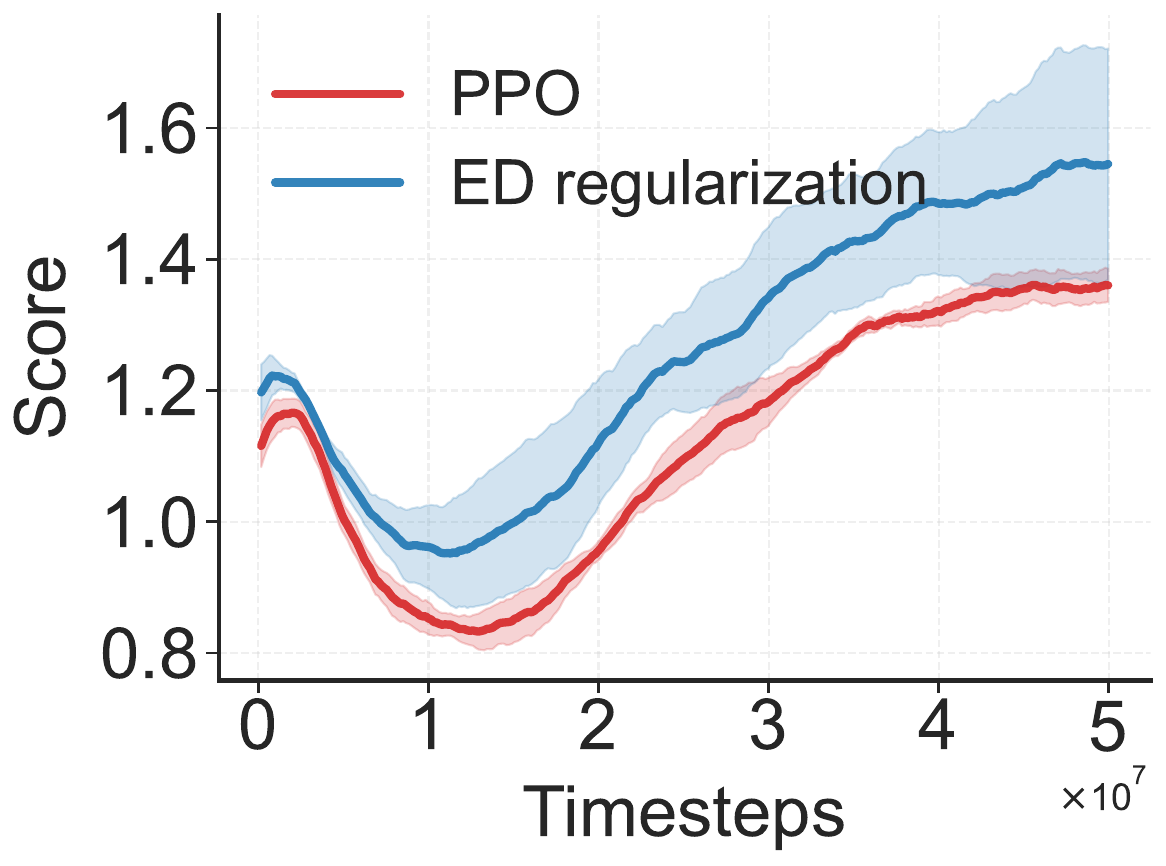}\vspace{-0.2em}
        \caption{\emph{Dodgeball}}
        \label{fig:Dodgeball}
    \end{subfigure}
    \hfill
    \begin{subfigure}[b]{0.22\textwidth}
        \centering
        \includegraphics[width=\linewidth]{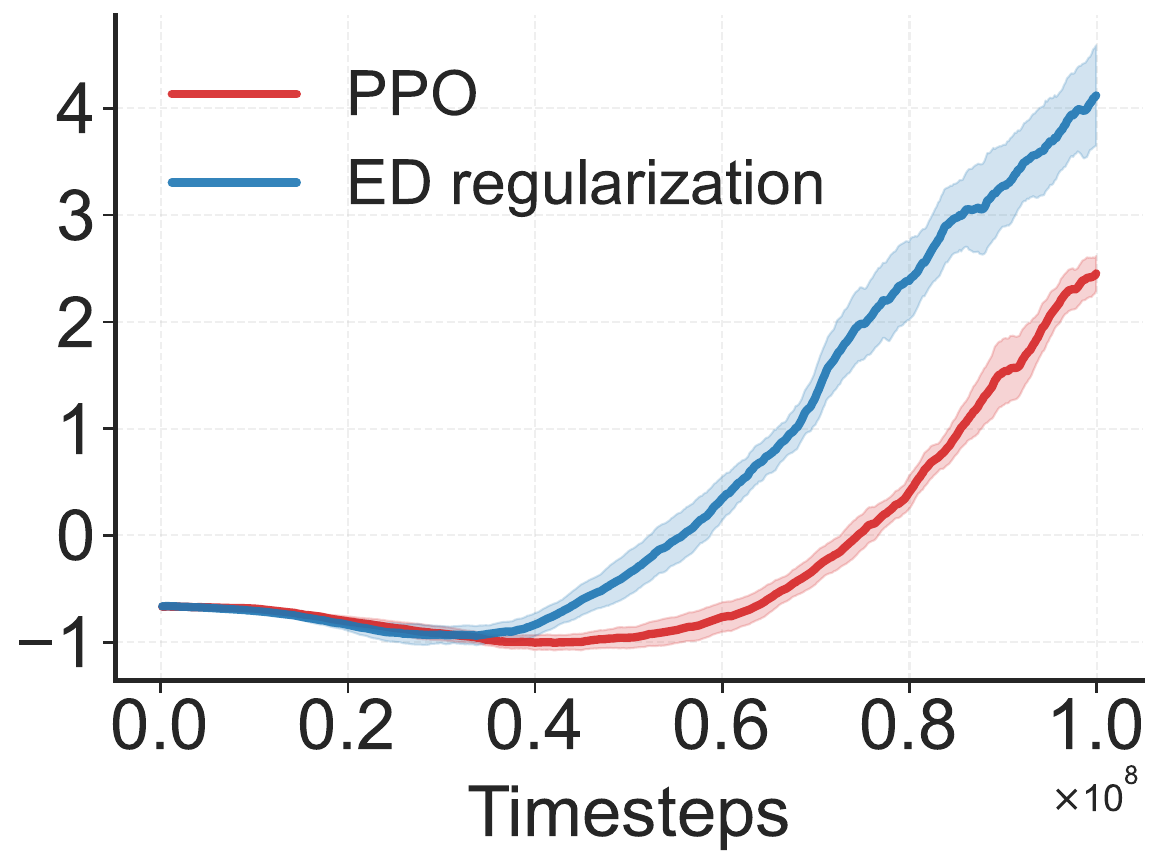}\vspace{-0.2em}
        \caption{\emph{Fruitbot}}
        \label{fig:Fruitbot}
    \end{subfigure}
    \hfill
    \begin{subfigure}[b]{0.22\textwidth}
        \centering
        \includegraphics[width=\linewidth]{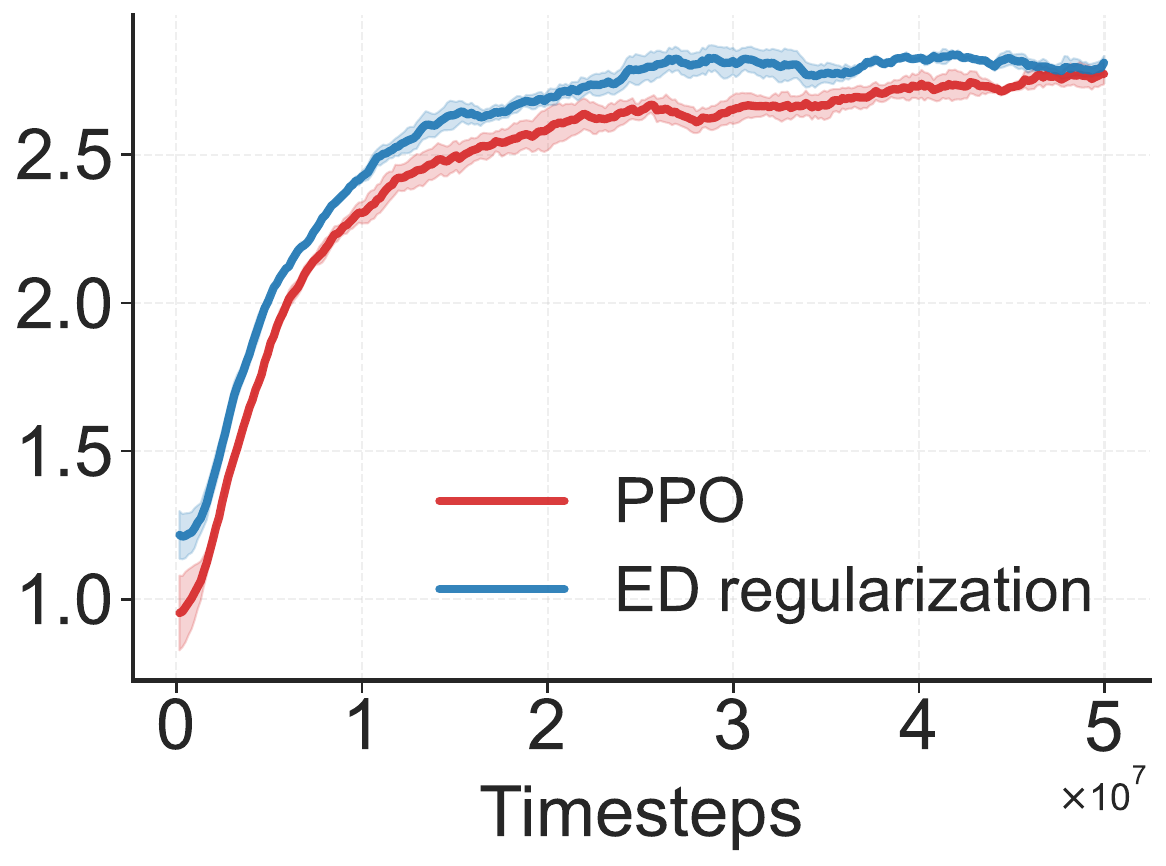}\vspace{-0.2em}
        \caption{\emph{Jumper}}
        \label{fig:Jumper}
    \end{subfigure}
    \hfill
    \begin{subfigure}[b]{0.22\textwidth}
        \centering
        \includegraphics[width=\linewidth]{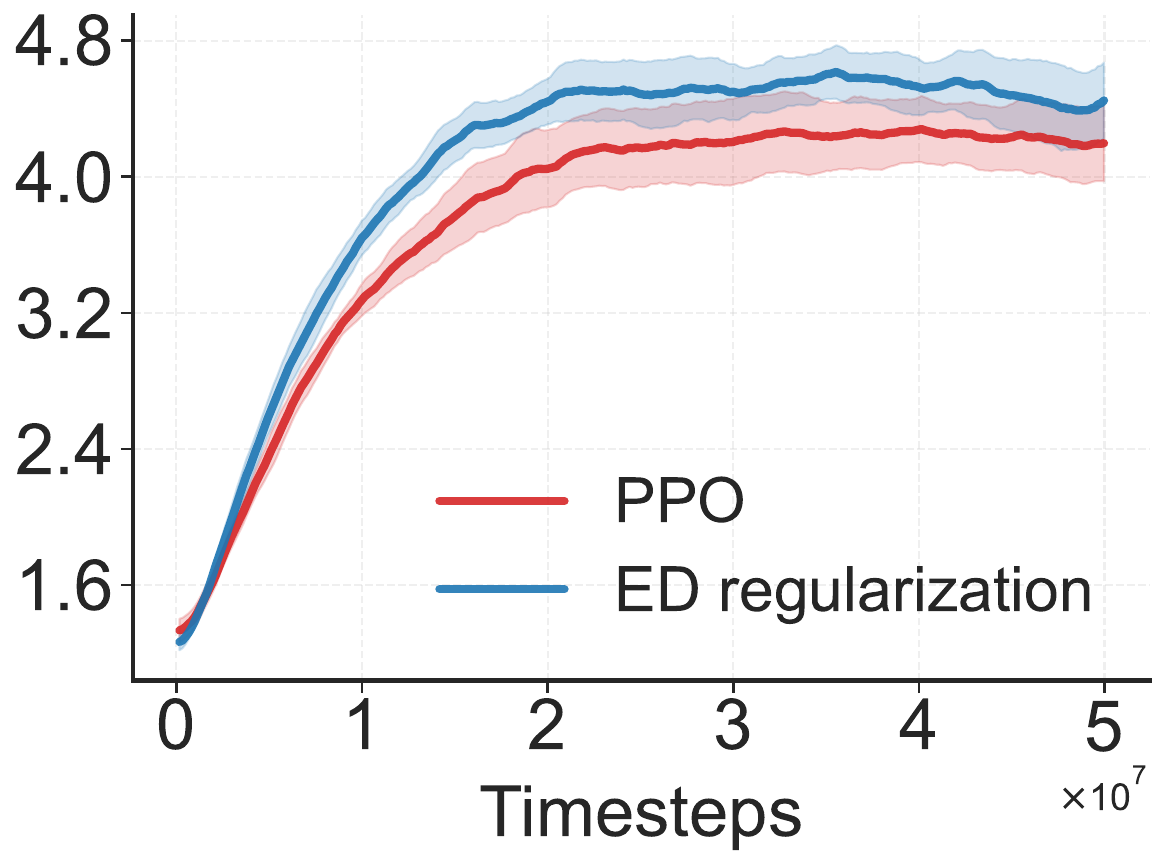}\vspace{-0.2em}
        \caption{\emph{StarPilot}}
        \label{fig:starpilot}
    \end{subfigure}\vspace{-0.2em}
    \caption{Generalization on unseen Procgen levels (averaged over 3 seeds). Shaded regions indicate standard errors of the mean.\vspace{-1.0em}}
    \label{fig:procgen_test}
\end{figure*}

We further evaluate the robustness of ED under a family of ImageNet-based distribution-shift benchmarks, including ImageNetV2~\citep{recht_imagenet_2019}, ImageNet-R~\citep{hendrycks_many_2021}, ImageNet-A~\citep{hendrycks_natural_2021}, ImageNet Sketch~\citep{wang_learning_2019}, and ObjectNet~\citep{barbu_objectnet_2019}.
Following~\citet{wortsman2022robust}, we end-to-end fine-tune CLIP ViT-B/16 and CLIP ViT-B/32 on ImageNet and apply \shortterm regularization during fine-tuning.
To incorporate a more advanced fine-tuning recipe, we also follow the weight-space ensembling strategy from \citet{wortsman2022robust} and interpolate between the weights of the zero-shot initialization ($\theta_{\text{ZS}}$) and the fine-tuned model ($\theta_{\text{FT}}$) using a mixing coefficient $\alpha \in [0, 1]$, such that $\theta_{\text{interp}} = (1-\alpha)\theta_{\text{ZS}} + \alpha\theta_{\text{FT}}$, and evaluate the performance of the interpolated model ($\theta_{\text{interp}}$).
Hyperparameters are listed in Appendix~\ref{sec:imagenet_finetune_settings}.

\textbf{Results.}
\Cref{tab:imagenet_results} summarizes the performance of \shortterm regularization compared to standard fine-tuning.
As observed, \shortterm consistently outperforms the baseline across both model architectures.
Notably, \shortterm improves performance not only on the in-distribution (ID) ImageNet validation set but also on all five \emph{out-of-distribution (OOD)} test sets.
These indicate that explicitly enforcing model simplicity enhances robustness against distribution shifts without compromising standard ID generalization capabilities.
For the weight-space ensembling setting,
\Cref{fig:robustness_curves} plots ImageNet accuracy versus average OOD accuracy across interpolation coefficients.
Across the full range, \shortterm-regularized models dominate the baseline trade-off, suggesting improved robustness.

\subsection{Text Classification on GLUE}
\label{subsec:glue}

We evaluate \shortterm regularization on three GLUE classification tasks~\citep{wang2018glue}, including two sentence-pair tasks (RTE, MRPC) and one acceptability task (CoLA), using BERT-base~\citep{devlin_bert_2018} and the standard GLUE evaluation metrics.
Unlike images, raw-token interpolation is ill-defined.
We thus construct interpolation paths in \emph{embedding space} and penalize functional complexity along these paths. Besides the standard BERT fine-tuning baseline, we also compare \shortterm with embedding mixup~\citep{guo2019augmenting}, which constructs mixed embeddings for pairs of examples and trains the classifier to match mixed labels. See Appendix~\ref{sec:nlp_settings} for training details.

\textbf{Results.}
\cref{tab:nlp_results} shows that \shortterm yields improvements across all three tasks, while embedding mixup is not consistently beneficial and can degrade performance. We posit that this is because enforcing \emph{linear} targets on the interpolation path as done by mixup is too stringent for language data. In contrast, \shortterm uses interpolation only to probe and penalize functional complexity rather than to enforce synthetic targets.

\subsection{Reinforcement Learning on Procgen}
\label{subsec:procgen}

Finally, we evaluate \shortterm regularization in reinforcement learning (RL) on the Procgen benchmark~\citep{cobbe_leveraging_2020}, a suite of procedurally generated environments designed to assess generalization capability of RL agents.
We train standard CNN-based PPO agents~\citep{schulman_proximal_2017} and apply the \shortterm penalty to the actor networks of the agents. Following standard protocols, we train agents on a fixed set of 500 hard-levels (train) and evaluate them on a distinct set of unseen levels (test). We report performance for four environments, including \emph{Dodgeball, Fruitbot, Jumper}, and \emph{StarPilot}. See Appendix~\ref{sec:rl_settings} for more details.

\begin{table}[t]
	\centering
	\setlength{\tabcolsep}{3pt}
	\small
	\caption{Performance comparison on each task using BERT-base (mean $\pm$ std over 9 seeds).}
	\label{tab:nlp_results}
	\begin{tabular}{lccc}
		\toprule
		\textbf{Method} & \textbf{RTE} & \textbf{MRPC} & \textbf{CoLA} \\
		& (Accuracy) & (Accuracy) & (Matthews Corr.) \\
		\midrule
		BERT-base & $ 70.28 \pm 1.73 $ & $ 86.74 \pm 1.06 $ & $ 62.31 \pm 1.01 $ \\
		\quad +Mixup & $ 70.28 \pm 3.07 $ & $86.30 \pm 1.22$ & $61.04 \pm 0.91$ \\
		\quad +\shortterm & $ \mathbf{71.12 \pm 1.96} $ & $ \mathbf{87.66 \pm 0.90} $ & $ \mathbf{62.45 \pm 1.01} $ \\
		\bottomrule
	\end{tabular}\vspace{-1.25em}
\end{table}

\textbf{Results.}
\Cref{fig:procgen_test} shows improved generalization across all four environments with \shortterm regularization, with higher asymptotic performance in \emph{Dodgeball, Fruitbot, StarPilot}, and faster learning in \emph{Jumper}. This indicates the applicability of simplicity regularization beyond supervised learning.

\subsection{Ablation, Robustness, and Overhead Analysis} 
\label{subsec:ablation}

\textbf{Protocol.}
Implementing ED involves configuring several hyperparameters. To reduce the burden of hyperparameter tuning and demonstrate the robustness of \shortterm, we first use CIFAR-10 as a primary testbed to analyze hyperparameter sensitivity and identify a robust default hyperparameter configuration; for other experiments, we \emph{fix} this unified configuration by tuning only the regularization strength $\lambda$, as described in Appendix~\ref{app:experimental_details_reg}. 

We then conduct several ablation studies to better understand the design choices underlying \shortterm. 

\textbf{Path construction.}
To verify the necessity of constructing \emph{data-dependent} interpolation paths, 
we show in \cref{app:ablation_random_pixels} that replacing real images with random noise slightly weakens the correlation and leads to poor regularization performance, highlighting the importance of using interpolation paths \emph{near the data manifold}.

We also examine whether \shortterm is restricted to input-space interpolation. As discussed in \cref{subsec:glue}, for discrete text inputs we already construct paths in embedding space; in \cref{app:feature-space}, we additionally evaluate intermediate feature-space interpolation for ViTs on CIFAR-10 and find that \shortterm still improves performance when applied after the embedding layer or after the first Transformer block.

\textbf{Polynomial fitting choices.}
We examine the choice of polynomial basis in \cref{app:ablation_basis_choice}, showing that replacing the Chebyshev basis with the Legendre basis yields similar performance, suggesting that \shortterm is stable across orthogonal polynomial bases rather than relying on a specific basis choice. Appendix~\ref{app:sampling-strategy}
compares randomized cosine sampling, fixed Chebyshev sampling, and
uniform sampling, showing that cosine-based sampling is substantially
more stable, especially at higher polynomial degrees.

\textbf{Output compression.}
Finally, we study whether the optional PCA-based output compression
step is responsible for the gains. In \cref{app:pca-dependence}, direct
polynomial fitting on the original outputs already works well, and
moderate PCA compression to two or three dimensions gives similar
performance, suggesting that the gains of \shortterm are not mainly
driven by PCA.

\textbf{Failure mode and efficiency.}
Beyond these controlled ablations, \cref{app:failure_analysis} provides
a failure analysis showing that \shortterm may fail when simpler features are easier to exploit but less desirable for robust generalization.
In \cref{sec:efficiency_analysis}, we provide a brief computational cost analysis and show that although \shortterm regularization increases training time due to the inclusion of additional interpolation examples, the overhead remains acceptable in our experiments.

\section{Conclusion, Limitations, and Future Work}

We introduced effective degree (ED), a general, function-space metric that quantifies the simplicity of a neural network through a polynomial surrogate.
Beyond using ED as a post-hoc generalization proxy, we derive analytic gradients through the fitting procedure and develop a numerically stable implementation, enabling ED to be used as an explicit simplicity regularizer during training.

Several directions remain open.
On the theory side, it would be valuable to formalize when path-based polynomial surrogates may preserve relevant notions of functional simplicity beyond polynomial degree.
On the methodology side, we plan to study alternative bases and sampling schemes, adaptive choices of surrogate degree and resolution, and more efficient estimators that reduce the overhead of regularization at scale.
Finally, we will explore ED regularization in additional settings where generalization is brittle, such as long-horizon sequence modeling and out-of-distribution detection, and investigate how ED interacts with pre-training.


\section*{Acknowledgements}
This work was supported in part by the National Key Research and Development Program of China No. 2024YDLN0006, and in part by the National Key Research and Development Program of China under STI 2030--Major Projects No. 2021ZD0200300, and in part by the Tsinghua--Fuzhou Data Technology Joint Research Institute (Project No. JIDT2024013).

\section*{Impact Statement}
This paper presents work whose goal is to advance the field of machine learning. There are many potential societal consequences of our work, none of which we feel must be specifically highlighted here.

\bibliography{ref}
\bibliographystyle{icml2026}

\newpage
\clearpage
\appendix
\onecolumn
\crefalias{section}{appendix}
\crefalias{subsection}{appendix}

\section{Proofs of Theoretical Results}

\subsection{Proof of Theorem~\ref{thm:path_preserve_complexity}}
\label{appsubsec:proof_theorem_complexity}

We first prove the following lemma.

\begin{lemma}[Almost-sure polynomial degree preservation under single interpolation paths]
\label{lem:degree_preservation}
Let $P:\mathbb{R}^d\to\mathbb{R}$ be a nonzero multivariate polynomial of degree $D\ge 1$.
Write its homogeneous decomposition as
$P(\mathbf{x})=\sum_{k=0}^{D} P_k(\mathbf{x})$,
where each $P_k$ is homogeneous of degree $k$ and $P_D\not\equiv 0$.
Let $\mathbf{x}_1,\mathbf{x}_2\stackrel{\mathrm{i.i.d.}}{\sim}\mathcal{D}$, and consider the interpolation path
\[
\mathbf{x}(\alpha)=\alpha\mathbf{x}_1+(1-\alpha)\mathbf{x}_2,\qquad \alpha\in[0,1].
\]
Assume that there exists a nonempty open set $U\subset\mathbb{R}^d$ such that
(i) $\mathbb{P}(\mathbf{x}\in U)=1$ for $\mathbf{x}\sim\mathcal{D}$ and
(ii) $\mathcal{D}$ has a density with respect to Lebesgue measure on $U$.
Then, with probability one,
\begin{equation}
\deg_{\alpha} P(\mathbf{x}(\alpha)) = D.
\end{equation}
Equivalently, for the random variable
\[
d_P(\mathbf{x}_1,\mathbf{x}_2)\coloneqq \deg_{\alpha}\, P\bigl(\alpha\mathbf{x}_1+(1-\alpha)\mathbf{x}_2\bigr),
\]
we have
\begin{equation}
    \mathbb{P}\left(d_P(\mathbf{x}_1,\mathbf{x}_2)=D\right)=1.
\end{equation}
\end{lemma}

\begin{proof}[Proof of Lemma~\ref{lem:degree_preservation}]
Fix $\mathbf{x}_2\in\mathbb{R}^d$ and define the direction $\mathbf{v}=\mathbf{x}_1-\mathbf{x}_2$.
Then the path can be rewritten as $\mathbf{x}(\alpha)=\mathbf{x}_2+\alpha\mathbf{v}$, and the path restriction is the univariate polynomial
\[
q_{\mathbf{x}_2,\mathbf{v}}(\alpha)\coloneqq P(\mathbf{x}_2+\alpha\mathbf{v}).
\]
We first identify the $\alpha^D$ coefficient of $q_{\mathbf{x}_2,\mathbf{v}}$.

Using the homogeneous decomposition $P=\sum_{k=0}^D P_k$, we expand
\[
P(\mathbf{x}_2+\alpha\mathbf{v})=\sum_{k=0}^D P_k(\mathbf{x}_2+\alpha\mathbf{v}).
\]
Since $P_k$ is homogeneous of degree $k$, the highest power of $\alpha$ appearing in
$P_k(\mathbf{x}_2+\alpha\mathbf{v})$ is $\alpha^k$, and its $\alpha^k$ coefficient equals $P_k(\mathbf{v})$.
Therefore,
\[
P(\mathbf{x}_2+\alpha\mathbf{v})
=\sum_{k=0}^D\Bigl(\alpha^k P_k(\mathbf{v})+\text{(lower-order terms in $\alpha$)}\Bigr),
\]
so the coefficient of $\alpha^D$ in $q_{\mathbf{x}_2,\mathbf{v}}(\alpha)$ is exactly $P_D(\mathbf{v})$.
Hence
\[
\deg_{\alpha} P(\mathbf{x}_2+\alpha\mathbf{v})=D
\quad\Longleftrightarrow\quad
P_D(\mathbf{v})\neq 0.
\]

It remains to show $\mathbb{P}(P_D(\mathbf{v})=0)=0$ for $\mathbf{v}=\mathbf{x}_1-\mathbf{x}_2$.
Because $P_D$ is a nonzero polynomial, its zero set
$Z\coloneqq\{\mathbf{v}\in\mathbb{R}^d:\;P_D(\mathbf{v})=0\}$
has Lebesgue measure zero.

Under the assumption that $\mathbf{x}_1,\mathbf{x}_2\in U$ almost surely and that $\mathcal{D}$ has a density $f$ on $U$,
the difference $\mathbf{v}=\mathbf{x}_1-\mathbf{x}_2$ is absolutely continuous on $U-U$ with density
\[
g(\mathbf{v})=\int_{\mathbb{R}^d} f(\mathbf{u})\,f(\mathbf{u}-\mathbf{v})\,d\mathbf{u},
\]
(where $f$ is extended by $0$ outside $U$).
In particular, $\mathbf{v}$ is absolutely continuous with respect to Lebesgue measure, and thus assigns probability zero to any Lebesgue-null set.
Therefore $\mathbb{P}(\mathbf{v}\in Z)=0$, i.e., $\mathbb{P}(P_D(\mathbf{v})=0)=0$.

Combining the above yields $\deg_{\alpha}P(\mathbf{x}(\alpha))=D$ almost surely, proving the claim.
\end{proof}

Now we are ready to prove Theorem~\ref{thm:path_preserve_complexity}.

\begin{proof}
By Lemma~\ref{lem:degree_preservation}, for each $i\in\{1,2\}$ we have
$d_{P_i}(\mathbf{x}_1,\mathbf{x}_2)=D_i$ almost surely.
Therefore $\mathbb{E}[d_{P_i}]=D_i$, and in particular $D_1>D_2$ implies $\mathbb{E}[d_{P_1}]>\mathbb{E}[d_{P_2}]$.

For the empirical statement, note that $0\le d_{P_i}(\mathbf{x}_1,\mathbf{x}_2)\le D_i$, so $d_{P_i}$ is integrable.
Since the pairs $(\mathbf{x}_1^{(j)},\mathbf{x}_2^{(j)})$ are i.i.d., the strong law of large numbers gives
\[
\widehat d_n(P_i)=\frac{1}{n}\sum_{j=1}^n d_{P_i}\left(\mathbf{x}_1^{(j)},\mathbf{x}_2^{(j)}\right)
\xrightarrow{\mathrm{a.s.}}
\mathbb{E}[d_{P_i}]
=D_i.
\]
If $D_1>D_2$, then almost surely there exists $N$ such that for all $n\ge N$,
$\widehat d_n(P_1)>\widehat d_n(P_2)$, because both sequences converge almost surely to distinct limits.
\end{proof}

\subsection{Proof of Proposition~\ref{prop:diff_ed} and Gradients of Damped Least Squares}
\label{appsubsec:proof_prop_diff}

\begin{proof}
    Our first result regarding the derivative with respect to coefficients,
    \[
    \frac{\partial \mathrm{\shortterm}}{\partial \mathbf{c}} = \mathrm{sign}(\mathbf{c}) \odot \mathbf{d},
    \]
    follows from standard differentiation rules, noting that the subgradient of $|c_k|$ is $\mathrm{sign}(c_k)$ for $c_k \neq 0$.
    
    In the standard polynomial fitting process, we estimate the coefficients $\mathbf{c}$ by solving the ordinary least-squares problem
    \[
    \mathbf{c}
    =
    \arg\min_{\mathbf{c}}
    \|\mathbf{T}\mathbf{c}-\mathbf{y}\|_2^2.
    \]
    When $\mathbf{T}^{\top}\mathbf{T}$ is invertible, the normal-equation solution is
    \[
    \mathbf{c}
    =
    (\mathbf{T}^{\top}\mathbf{T})^{-1}\mathbf{T}^{\top}\mathbf{y}.
    \]
    Differentiating both sides with respect to $\mathbf{y}$ yields the Jacobian matrix:
    \[
    \frac{\partial \mathbf{c}}{\partial \mathbf{y}} = (\mathbf{T}^\top\mathbf{T})^{-1}\mathbf{T}^\top.
    \]
    Finally, applying the chain rule yields the gradient presented in Proposition~\ref{prop:diff_ed}:
    \[
    \frac{\partial \mathrm{\shortterm}}{\partial \mathbf{y}}
    =
    \left(\frac{\partial \mathbf{c}}{\partial \mathbf{y}}\right)^\top
    \frac{\partial \mathrm{\shortterm}}{\partial \mathbf{c}}
    =
    \mathbf{T}(\mathbf{T}^\top \mathbf{T})^{-1}
    \left( \mathrm{sign}(\mathbf{c}) \odot \mathbf{d} \right).
    \]
\end{proof}

\paragraph{Gradient for the stable implementation (damped least squares).}
In Section~\ref{subsec:diff_ed}, we introduce a damping factor $\epsilon > 0$ to enhance numerical stability. This corresponds to solving the regularized least squares problem (Ridge Regression), where the coefficients $\mathbf{c}_\epsilon$ satisfy:
\[
    (\mathbf{T}^\top\mathbf{T} + \epsilon \mathbf{I}) \mathbf{c}_\epsilon = \mathbf{T}^\top \mathbf{y}.
\]
This linear system has the closed-form solution:
\begin{equation}
\label{eq:damped_c_sol}
    \mathbf{c}_\epsilon = (\mathbf{T}^\top\mathbf{T} + \epsilon \mathbf{I})^{-1} \mathbf{T}^\top \mathbf{y}.
\end{equation}
Differentiating both sides with respect to $\mathbf{y}$:
\[
    \frac{\partial \mathbf{c}_\epsilon}{\partial \mathbf{y}} = (\mathbf{T}^\top\mathbf{T} + \epsilon \mathbf{I})^{-1} \mathbf{T}^\top.
\]
Applying the chain rule again, the gradient used in our stable implementation is:
\begin{equation}
\label{eq:damped_grad_final}
    \frac{\partial \mathrm{\shortterm}}{\partial \mathbf{y}} 
    = \left(\frac{\partial \mathbf{c}_\epsilon}{\partial \mathbf{y}}\right)^\top \frac{\partial \mathrm{\shortterm}}{\partial \mathbf{c}_\epsilon}
    = \mathbf{T} (\mathbf{T}^\top\mathbf{T} + \epsilon \mathbf{I})^{-1} (\mathrm{sign}(\mathbf{c}_\epsilon) \odot \mathbf{d}).
\end{equation}
This confirms that the gradient computation remains valid and analytically tractable when using the damped solver. In the limit $\epsilon \to 0$, Eq.~\eqref{eq:damped_grad_final} recovers the result in Proposition~\ref{prop:diff_ed}.

\section{Additional Correlation Results}
\label{app:additional_correlations}

This appendix provides supplementary results for the correlation analysis discussed in \cref{subsec:correlation_gap}, including detailed results for ViT-Tiny on CIFAR-10 and CLIP models trained without mixup on ImageNet.

\subsection{ViT-Tiny on CIFAR-10}
\label{app:vit_corr}

In the main text, we summarized the correlation analysis for ViT-Tiny. Here, we provide the detailed experimental setup and the corresponding plots.

\paragraph{Training setup.}
We evaluated ViT-Tiny models using a grid of 27 hyperparameter configurations adapted for Transformer architectures. Consistent with the ResNet18 protocol, we report results averaged over three random seeds for each configuration.

\paragraph{Results.}
As shown in \cref{fig:vit_tiny_all_app}, \longterm maintains a strong positive correlation with the generalization gap across diverse hyperparameter settings.
In contrast, sharpness-based measures exhibit weaker predictive power, while the parameter $L_2$ norm correlates negatively with the generalization gap, failing to capture the correct complexity-generalization relationship in this context.

\begin{figure}[h]
	\centering
	\includegraphics[width=\textwidth]{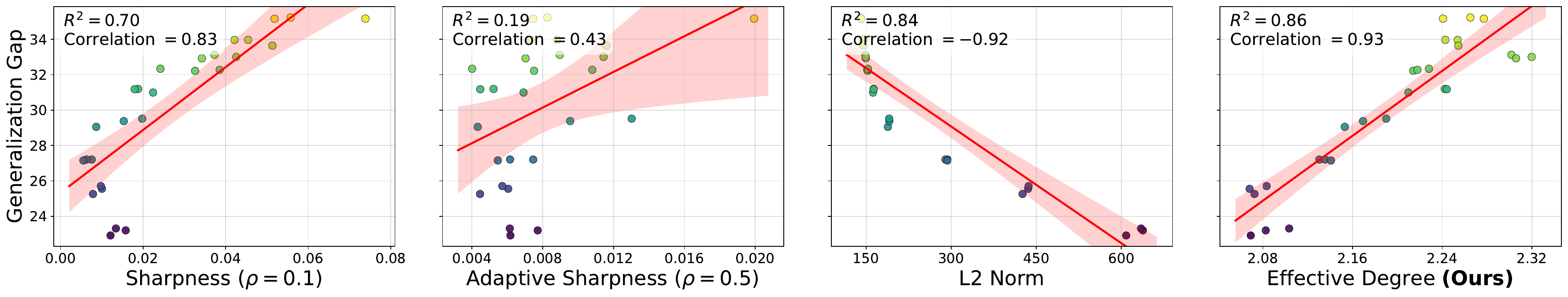}
	\caption{Correlation between \longterm, sharpness-based measures, and parameter $L_2$ norm with generalization gap for ViT-Tiny on CIFAR-10.
		The four panels (left to right) plot generalization gap against standard sharpness, adaptive sharpness, parameter $L_2$ norm, and \longterm, respectively.
		Each point corresponds to the average over three random seeds under a specific hyperparameter configuration.
	}
	\label{fig:vit_tiny_all_app}
\end{figure}

\subsection{CLIP Models on ImageNet without Mixup}
\label{app:imagenet_corr}

We also present the analysis on ImageNet to CLIP models fine-tuned \emph{without} mixup augmentation.
\Cref{fig:imagenet_clip_nomixup} shows the correlation plots for this setting.
Consistent with the mixup-trained models reported in the main text (\cref{fig:imagenet_clip_mixup_active}), \longterm maintains a strong positive correlation with the generalization gap.
Sharpness-based measures continue to exhibit negative correlations, while the parameter $L_2$ norm shows only a weak relationship with generalization performance.

\begin{figure}[h]
	\centering
	\includegraphics[width=\textwidth]{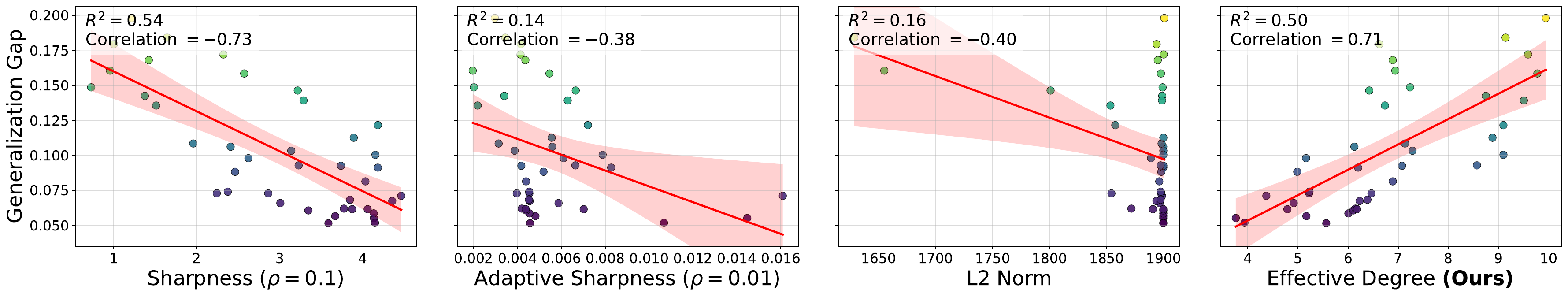}
	\caption{Correlation plots for CLIP models trained without mixup on ImageNet.
		The four panels (left to right) plot generalization gap against standard sharpness, adaptive sharpness, parameter $L_2$ norm, and \longterm, respectively.}
	\label{fig:imagenet_clip_nomixup}
\end{figure}

\section{Ablation Study}
\label{app:ablation}

\subsection{Replacing Real Images with Synthetic Random Images in Constructing Interpolation Paths}
\label{app:ablation_random_pixels}

We study whether the effectiveness of \longterm depends on using \emph{real} data as interpolation endpoints.
To this end, we replace the real images of CIFAR-10 used to construct \longterm with i.i.d. sampled uniform noise on each pixel, referred to as random pixels, and evaluate both (i) the correlation between \longterm and the generalization gap and (ii) test accuracy under regularized training.
For the correlation analysis, we use ResNet18 following the setup in the main text; for regularized training, we use ViT-Tiny with the same training protocol, hyperparameter configurations, and three-seed averaging as in the main text.

\paragraph{Correlation.}
As shown in \cref{fig:resnet18_random_replacement}, replacing real images with random pixels slightly weakens the correlation with the generalization gap on ResNet-18, although the overall positive trend remains.

\begin{figure}[h]
	\centering
	\includegraphics[width=0.3\textwidth]{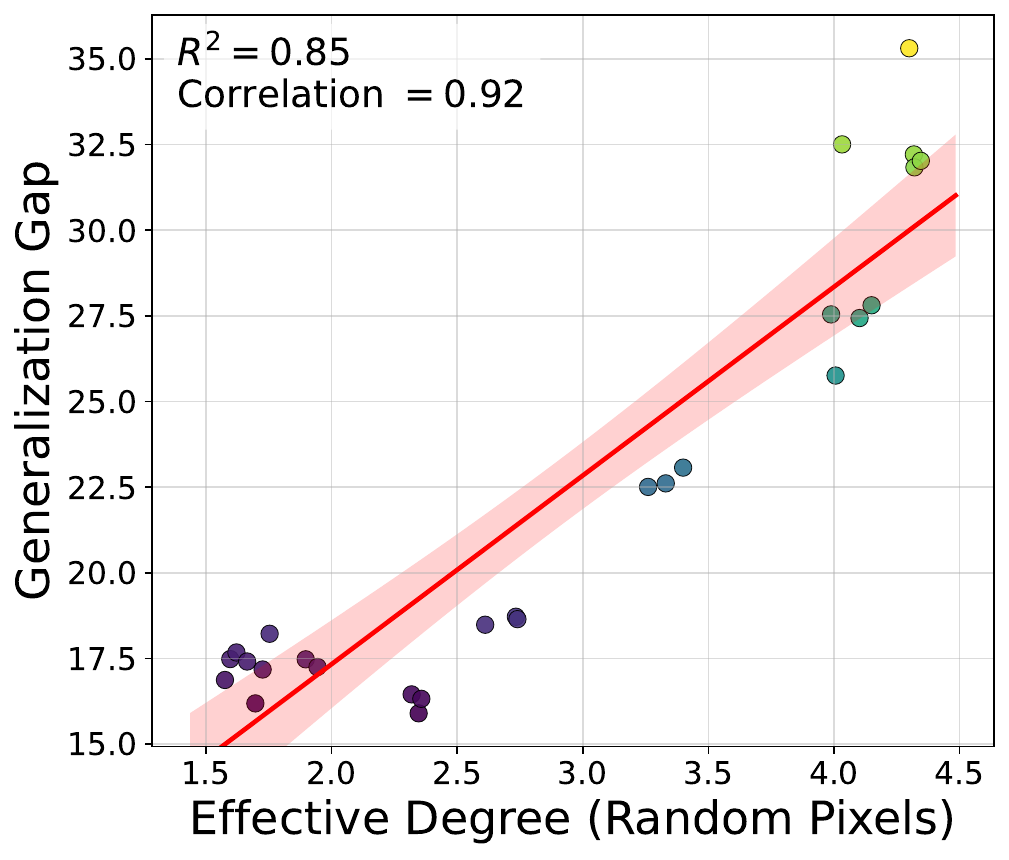}
	\caption{Correlation between \longterm computed with uniformly sampled random pixels and the generalization gap on ResNet-18. Points are averaged over three seeds.}
	\label{fig:resnet18_random_replacement}
\end{figure}

\paragraph{ED regularization.}
\Cref{tab:ablation_random_images_acc} reports the top-1 accuracy on CIFAR-10 for ViT-Tiny under the same explicit regularization setting as in the main text.
Notably, \shortterm with random pixels completely removes the gain of \shortterm, indicating that computing \shortterm using on-distribution samples to construct interpolation paths \emph{near the data manifold} is crucial for its effectiveness.

\begin{table}[t]
	\centering
	\small
	\caption{Top-1 test accuracy on CIFAR-10 (\%, mean $\pm$ std over 3 seeds) for ViT-Tiny under the same explicit regularization setting as in the main text. We compare the Baseline, \shortterm, and an ablation \shortterm{} (random pixels) that replaces real interpolation endpoints with uniformly sampled random images.}
	\label{tab:ablation_random_images_acc}
	\begin{tabular}{lccc}
		\toprule
		Model & Baseline & \shortterm & \shortterm{} (random pixels) \\
		\midrule
		ViT-Tiny
		& $87.80 \pm 1.17$
		& $\mathbf{90.82 \pm 0.11}$
		& $87.31 \pm 0.80$ \\
		\bottomrule
	\end{tabular}
\end{table}

\subsection{Choice of Polynomial Basis}
\label{app:ablation_basis_choice}

We next examine whether the effectiveness of \shortterm depends on the specific choice of polynomial basis.
We use the Chebyshev basis primarily for its numerical stability in polynomial fitting, while viewing the basis as a tractable surrogate for representing the path-restricted function.
Since our simplicity notion is degree-based, orthogonal polynomial bases are natural choices.
By contrast, non-polynomial bases such as Fourier or wavelet bases would induce related but conceptually different notions of spectral simplicity.

To directly test the dependence on the basis choice, we replace the Chebyshev basis with the Legendre basis on CIFAR-10, while keeping the same hyperparameters under the efficiency-oriented configuration and using three random seeds.
The results are reported in \cref{tab:ablation_basis_choice}.

\begin{table}[h]
	\centering
	\small
	\caption{Top-1 test accuracy on CIFAR-10 (\%, mean $\pm$ std over 3 seeds) for ViT-Tiny with different orthogonal polynomial bases. The small difference between Chebyshev and Legendre bases suggests that \shortterm is stable across orthogonal polynomial bases.}
	\label{tab:ablation_basis_choice}
	\begin{tabular}{lc}
		\toprule
		Basis & Test accuracy \\
		\midrule
		Chebyshev & $90.18 \pm 0.14$ \\
		Legendre & $89.89 \pm 1.29$ \\
		\bottomrule
	\end{tabular}
\end{table}

The effect of changing the basis is small.
This suggests that the method is stable across orthogonal polynomial bases, rather than relying specifically on Chebyshev polynomials.

\subsection{Linear Blending in Intermediate Feature Space}
\label{app:feature-space}

Our default construction uses input-space interpolation for continuous
inputs because it is simple and architecture-agnostic. However, \shortterm is
not restricted to input space. For discrete inputs such as text,
direct interpolation over raw tokens is not meaningful, and our
BERT-based experiments in \cref{subsec:glue} therefore construct paths in
embedding space. We further investigate whether \shortterm can be applied to
intermediate feature spaces for vision models.

Specifically, on CIFAR-10, we apply \shortterm regularization to ViT-Tiny
using interpolation paths constructed at three different locations:
the input space, the feature space after the embedding layer, and the
feature space after the first Transformer attention block. The results
are shown in Table~\ref{tab:feature-space-interpolation}.

\begin{table}[h]
\centering
\caption{Top-1 test accuracy on CIFAR-10 (\%) under different
interpolation spaces. \shortterm still improves performance when applied in
intermediate feature spaces, although input-space interpolation gives
the largest gain.}
\label{tab:feature-space-interpolation}
\begin{tabular}{lcccc}
\toprule
Method & Baseline w/o \shortterm & Input & Embedding & After block 1 \\
\midrule
Test accuracy & 86.49 & 89.99 & 88.92 & 88.18 \\
\bottomrule
\end{tabular}
\end{table}

\shortterm still improves performance in intermediate feature spaces, although the gain is smaller than with input-space interpolation. A plausible reason is that intermediate
feature-space interpolation constrains the end-to-end input-output
function less directly, even though the regularization signal still
backpropagates to earlier layers. That said, we note that our current implementation of feature-space interpolation is preliminary and believe it remains a promising direction to be further explored.

\subsection{Sampling Strategy}
\label{app:sampling-strategy}

We also study the effect of the sampling strategy used to construct points along each interpolation path. We compare three choices: fixed Chebyshev sampling from \cref{eq:cheb_nodes}, randomized cosine sampling from \cref{eq:cos_sample}, and uniform sampling. We evaluate both a low-resolution setting with resolution $r=4$ and maximum degree $K=3$, and a higher-degree setting with resolution $r=15$ and maximum degree $K=14$. The results are shown in Table~\ref{tab:sampling-ablation}.

\begin{table}[h]
\centering
\caption{Top-1 test accuracy on CIFAR-10 (\%) under different
sampling strategies. }
\label{tab:sampling-ablation}
\begin{tabular}{ccccc}
\toprule
Resolution $r$ & Max degree $K$ & Sampling & Test accuracy \\
\midrule
4  & 3  & Randomized cosine & 89.99 \\
4  & 3  & Fixed Chebyshev    & 89.47 \\
4  & 3  & Uniform            & 89.23 \\
15 & 14 & Randomized cosine & 89.90 \\
15 & 14 & Fixed Chebyshev    & 90.40 \\
15 & 14 & Uniform            & 78.42 \\
\bottomrule
\end{tabular}
\end{table}

At small resolution and degree, all three sampling strategies perform
reasonably well, with randomized cosine sampling giving the best result
among the compared settings. In the higher-degree setting, however,
uniform sampling becomes unstable and leads to a substantial accuracy
drop. In contrast, the two cosine-based strategies remain effective.

\subsection{Dependence on PCA-Based Output Compression}
\label{app:pca-dependence}

We further study whether the effectiveness of \shortterm depends on the
PCA-based output compression step. In our method, PCA is an optional
per-path output compression module used mainly for computational
efficiency when the model output is high-dimensional, rather than the
source of the \shortterm regularization effect.

To test this, we conduct a CIFAR-10 ablation under the
efficiency-oriented setting, comparing direct polynomial fitting on the
original model outputs without PCA against PCA projections with
different output dimensions. The results are reported in
Table~\ref{tab:pca-dependence}.

\begin{table}[h]
\centering
\caption{Top-1 test accuracy on CIFAR-10 (\%) under different
PCA output dimensions. ``No PCA'' denotes direct fitting on the
original model outputs.}
\label{tab:pca-dependence}
\begin{tabular}{lcccc}
\toprule
Output dimension after PCA & No PCA & 3D & 2D & 1D \\
\midrule
Test accuracy & 89.99 & 89.84 & 89.90 & 88.75 \\
\bottomrule
\end{tabular}
\end{table}

These results suggest that \shortterm is not mainly driven by PCA. Direct
fitting on the original outputs already achieves strong performance,
and PCA compression to two or three dimensions remains very close to
the no-PCA result. Only aggressive compression to one dimension
noticeably hurts performance, indicating that overly low-dimensional
projections may discard useful output variation along the interpolation
paths.

\section{Experimental Details for Generalization Analysis}
\label{app:experimental_details}

This appendix details the experimental protocols specifically for the correlation and grokking analyses presented in \cref{sec:generalization_prediction}.

\subsection{Generalization Prediction on CIFAR-10 and ImageNet}
\label{app:exp_correlation}

We describe the generation of the model pool and the specific protocols for estimating \longterm in the correlation experiments.

\paragraph{Model pool generation.}
To evaluate the correlation between complexity and generalization, we trained a diverse set of models on CIFAR-10 by sweeping over key hyperparameters.
For \textbf{ResNet18}, following \citet{li_understanding_2025}, we train models using SGD with a MultiStep learning rate schedule, sweeping over batch sizes $\{256, 512, 1024\}$, learning rates $\{0.1, 0.01, 0.001\}$, and weight decays $\{10^{-5}, 10^{-6}, 10^{-7}\}$. This yields 27 unique configurations, each averaged over three random seeds.
For \textbf{ViT-Tiny}, we employ the AdamW optimizer~\citep{loshchilov_decoupled_2019} with cosine learning rate decay, sweeping batch sizes $\{256, 512, 1024\}$, learning rates $\{0.005, 0.001, 0.0005\}$, and weight decays $\{10^{-3}, 10^{-4}, 10^{-5}\}$.

\paragraph{Baseline configurations.}
For sharpness-based baselines, we report the best correlation achieved across a range of neighborhood sizes $\rho$.
For standard sharpness, we sweep $\rho \in \{0.01, 0.05, 0.1\}$. For adaptive sharpness, we sweep $\rho \in \{0.01, 0.05, 0.1, 0.5, 1.0\}$.
On CIFAR-10, sharpness is computed using the full training set; on ImageNet, we use a subset of 2,048 samples following \citet{silva_hide_2025}.

\paragraph{Estimation protocol of effective degree.}
For the correlation analysis, we employ a high-precision fitting protocol to minimize variance. For both CIFAR-10 and ImageNet experiments, we average results over 400 independent estimations, with the resolution set to 200 and the maximum polynomial degree fixed at 40.
We report the best correlation achieved over metric variants (raw vs.\ normalized \shortterm).
Specifically, for the raw variant, we fit the polynomial to the model's output after Softmax (as it provides inherent normalization), whereas for the normalized variant, we fit the logits directly, with normalization explicitly handled within the \shortterm calculation.
For ImageNet models (1,000 classes), we additionally use PCA to project outputs onto lower-dimensional subspaces ($m \in \{1, 3, 10\}$) prior to fitting, reporting the highest correlation across these configurations.

\subsection{Grokking Dynamics}
\label{app:exp_grokking}

We investigate grokking using the Modular Division task over $\mathbb{Z}_{97}$. The model is a 2-layer Transformer trained on a 30\% subset of the data using AdamW (learning rate $10^{-3}$, weight decay 0), following the setup of \citet{power_grokking_2022}.

\paragraph{Complexity tracking.}
We monitor \longterm and baselines throughout the training trajectory.
For \longterm estimation during grokking, we sample $n_p=200$ direction pairs to ensure stability. For the polynomial approximation along each 1D slice, we set the resolution to 64 and the maximum degree to 40. This configuration balances computational efficiency with the precision required to track the phase transition.

\section{\shortterm Regularization Pseudocode}
\label{appsec:algo}

We provide the pseudocode for our proposed \shortterm regularization training scheme in \cref{alg:shortterm_regularization}.

\begin{algorithm}[t]
  \caption{Training with \longterm (ED) regularization}
  \label{alg:shortterm_regularization}
  \begin{algorithmic}[1]
    \renewcommand\algorithmicrequire{\textbf{Input:}}
    \renewcommand\algorithmicensure{\textbf{Output:}}
    
    \REQUIRE
      Model $f_\theta$; training data $\mathcal{D}$;
      regularization strength $\lambda$;
      number of sampled pairs $n_p$;
      sampling resolution $r$; maximum polynomial degree $K$;
      jitter $\epsilon$; 
      \textbf{optional:} target PCA dim $m$. 
    \ENSURE Trained parameters $\theta$.

    \FOR{each minibatch $\mathcal{B}=\{(\mathbf{x}_b, \mathbf{t}_b)\}_{b=1}^B$}
      \STATE Compute $\mathcal{L}_{\mathrm{task}}(\theta;\mathcal{B})$.
      \STATE $\widehat{\mathrm{\shortterm}}_{\mathcal{B}} \gets 0$. \COMMENT{initialize regularizer estimate}

      \FOR{$p \gets 1$ to $n_p$}
        \STATE Sample $(\mathbf{x}_1, \mathbf{t}_1), (\mathbf{x}_2, \mathbf{t}_2)$ uniformly from $\mathcal{B}$. \COMMENT{require labels $\mathbf{t}_1, \mathbf{t}_2$ to use label-anchoring}
        \STATE Sample nodes $\{\alpha_\ell\}_{\ell=1}^r \subset [0,1]$ via a sampling scheme. \COMMENT{Chebyshev nodes or randomized cosine sampling}
        
        \STATE Set $\mathbf{x}(\alpha_\ell) \gets \alpha_\ell \mathbf{x}_1 + (1-\alpha_\ell) \mathbf{x}_2$ for $\ell=1,\dots,r$. \COMMENT{linear interpolation between $\mathbf{x}_1$ and $\mathbf{x}_2$}
        \STATE $\mathbf{y}_\ell \gets f_\theta(\mathbf{x}(\alpha_\ell))$, for $\ell=1,\dots,r$.

        \STATE \textbf{If} label-anchored: set $\mathbf{y}_1 \gets \mathbf{t}_2$ and $\mathbf{y}_r \gets \mathbf{t}_1$. \COMMENT{optional: apply label-anchoring on classification tasks}

        \STATE \textbf{If} use PCA: project $\{\mathbf{y}_\ell\}_{\ell=1}^r$ to $\mathbb{R}^{m}$ and set $n \gets m$; \textbf{else} set $n \gets \dim(\mathbf{y}_1)$. \COMMENT{optional: output reduction}

        \FOR{$j \gets 1$ to $n$}
          \STATE $\mathbf{c}^{(j)} \gets \texttt{LinearSolve}(\mathbf{T}^\top\mathbf{T} + \epsilon \mathbf{I},\, \mathbf{T}^\top \mathbf{y}^{(j)})$. \COMMENT{solve coefficients}
          \STATE $\mathrm{\shortterm}_p^{(j)} \gets \sum_{k=0}^{K} \lvert c^{(j)}_k\rvert \, k$. \COMMENT{\longterm\ for $j$-th dim}
        \ENDFOR

        \STATE $\mathrm{\shortterm}_p \gets \frac{1}{n}\sum_{j=1}^{n} \mathrm{\shortterm}_p^{(j)}$. \COMMENT{average over $n$ dimensions}
        \STATE $\widehat{\mathrm{\shortterm}}_{\mathcal{B}} \gets \widehat{\mathrm{\shortterm}}_{\mathcal{B}} + \mathrm{\shortterm}_p$. 
      \ENDFOR

      \STATE $\widehat{\mathrm{\shortterm}}_{\mathcal{B}} \gets \widehat{\mathrm{\shortterm}}_{\mathcal{B}} / n_p$. \COMMENT{average over sampled pairs}
      \STATE $\mathcal{L}(\theta;\mathcal{B}) \gets \mathcal{L}_{\mathrm{task}}(\theta;\mathcal{B}) + \lambda\,\widehat{\mathrm{\shortterm}}_{\mathcal{B}}$.
      \STATE Update $\theta$ via gradient descent on $\mathcal{L}(\theta;\mathcal{B})$. 
    \ENDFOR

    \STATE \textbf{return} $\theta$.
  \end{algorithmic}
\end{algorithm}

\section{Experimental Details for \shortterm Regularization}
\label{app:experimental_details_reg}
In this section, we detail the experimental settings for the baselines and the configurations for \shortterm regularization.

To ensure a fair comparison and reproducibility, our experimental settings adhere to the following principles:
\begin{itemize}
    \item \textbf{Baselines:} We strictly follow the hyperparameter configurations reported in the original papers or official open-source implementations. For baselines implemented by us, we perform a standard grid search to ensure optimal performance.
    \item \textbf{\shortterm regularization:} Implementing \shortterm involves configuring several hyperparameters: the regularization strength $\lambda$, sampling resolution $r$, maximum polynomial degree $K$, number of input pairs $n_p$, and the projection dimension $m$ (when using PCA). To reduce the burden of hyperparameter tuning and demonstrate robustness of our method, we first use CIFAR-10 as a primary testbed to analyze hyperparameter sensitivity and identify a robust default \textbf{efficiency-oriented configuration} (described in \cref{sec:cifar10_settings}). We then \emph{fix} these structural parameters except $\lambda$ for all subsequent experiments (ImageNet, language modeling, reinforcement learning). Specifically, we set the sampling resolution to $r=4$, the maximum polynomial degree to $K=3$, and the number of sampled pairs to half the batch size ($n_p = B/2$) or the full batch size. For tasks with high-dimensional output spaces (e.g., ImageNet, reinforcement learning), we apply PCA with a projected dimension of $m=3$ to mitigate computational overhead. Notably, we consistently enforce $m = r-1$, ensuring that $m$ serves as a dependent variable rather than an additional hyperparameter. Furthermore, given the resolution $r=4$, we utilize randomized cosine sampling (Eq.~\eqref{eq:cos_sample}) to enhance the diversity of the sampling points. As detailed in \cref{sec:efficiency_analysis}, this configuration offers a highly favorable computational profile.
    
    Moreover, we consistently employ the standard (unnormalized) effective degree $\mathrm{\shortterm}$ (Definition~\ref{def:effective_degree}) as the regularization penalty. 
    This choice is motivated by the fact that we fit post-softmax probabilities, which are inherently normalized and bounded; thus, the unnormalized ED provides a stable complexity measure without requiring additional scale invariance.
    Consequently, the only hyperparameter requiring tuning for new tasks is the regularization strength $\lambda$.
\end{itemize}

\subsection{Settings for CIFAR-10}
\label{sec:cifar10_settings}
The baseline, \shortterm regularization, mixup, SAM and ASAM share the same backbone, optimizer, learning-rate schedule, and preprocessing, and differ only in the corresponding training strategy.
We train all models using AdamW for 300 epochs, with batch size 256, base learning rate 0.005, and weight decay 0.1.
The learning rate is linearly warmed up for the first 10 epochs and then decayed with cosine annealing to a minimum learning rate of 0. Random cropping and horizontal flip data augmentations are used.

\paragraph{Mixup.}
For mixup, we sample the mixing coefficient $\lambda \sim \mathrm{Beta}(\alpha,\alpha)$ with $\alpha=1.0$, which reduces to a uniform distribution on $[0,1]$.
Given two training examples $(x_1,t_1)$ and $(x_2,t_2)$, mixup constructs a mixed sample
$\tilde{x}=\lambda x_1+(1-\lambda)x_2$ with the corresponding soft label
$\tilde{t}=\lambda t_1+(1-\lambda)t_2$.

\paragraph{SAM/ASAM.}
For sharpness-aware minimization (SAM) and its adaptive variant (ASAM), we follow the standard formulations and tune the neighborhood size $\rho$ using the grids from the original papers.
Specifically, for SAM we use $\rho \in \{0.01, 0.02, 0.05, 0.1, 0.2, 0.5\}$,
and for ASAM we use $\rho \in \{5\times10^{-5}, 10^{-4}, 2\times10^{-4}, \ldots, 0.5, 1.0, 2.0\}$.
For each method, we report the best result over the corresponding $\rho$ grid (selected by top-1 test accuracy).

\paragraph{Jacobian regularization.}
For Jacobian regularization (JR), we employ random projections to efficiently approximate the Jacobian norm.
We tune the regularization coefficient $\lambda_{\text{JR}}$ over the grid $\{0.01, 0.05, 0.1, 0.5, 1.0\}$.
Experiments are repeated with three random seeds, and we report the best result selected by top-1 test accuracy.

\paragraph{\shortterm regularization.}
For \shortterm regularization, we implement a ramp-up schedule for the regularization strength over the first 100 epochs, increasing sinusoidally from 0 to $\lambda$. This approach facilitates rapid task learning in the early stages of training before enforcing stronger complexity control. Additionally, we adopt the label-anchored \shortterm strategy described in Section~\ref{sec:label_anchored}. Regarding the structural hyperparameters, specifically resolution $r$, degree $K$, and the number of pairs $n_p$, we explore two distinct configurations to balance performance and efficiency:



\begin{itemize}
    \item \textbf{Performance-oriented configuration:} 
    In our initial search for optimal performance, we observe that higher sampling resolutions generally improve the stability and effectiveness of the regularizer with a fixed Chebyshev-node sampling scheme (Eq.~\eqref{eq:cheb_nodes}).
    The optimal setting is $r=15$, $K=7$ with regularization strength $\lambda=2$ and compute \shortterm using $n_p=256$ sampled pairs within each minibatch. 
    While these settings yield the best accuracy, they incur higher computational costs.

    \item \textbf{Efficiency-oriented configuration:} \label{itm:efficiency_config}
    To enhance practical applicability, we evaluate a lightweight configuration characterized by $r=4$, $K=3$, and $n_p=128$ (corresponding to half the batch size) with randomized cosine sampling. For this setting, the regularization strength was determined to be $\lambda=7$ based on empirical search. 
    Despite the reduced resolution, this setting achieves competitive results (as shown in Table~\ref{tab:ablation_between_resolution}) while significantly reducing training time. Consequently, this serves as a standard setting where resolution parameters are fixed, reducing the hyperparameter search space to only $\lambda$.
\end{itemize}

\begin{table}[t]
	\centering
	\small
	\caption{Top-1 test accuracy on CIFAR-10 (\%, mean $\pm$ std over 3 seeds) investigating the performance of the performance-oriented configuration and the efficiency-oriented configuration}
	\label{tab:ablation_between_resolution}
	\begin{tabular}{lcc}
		\toprule
		Model & Performance-oriented configuration & Efficiency-oriented configuration \\
		\midrule
		ViT-Tiny
		& $\mathbf{90.82 \pm 0.11}$
		& $90.18 \pm 0.14$ \\
		\bottomrule
	\end{tabular}
\end{table}

\subsection{Settings for ImageNet}
\label{sec:imagenet_settings}

For the ImageNet experiments, we employ the ViT-S/16 architecture and evaluate our method under two training recipes:
\begin{itemize}
    \item \textbf{Original recipe:} This protocol largely follows the standard recipe outlined in~\citet{dosovitskiy_image_2021}. To accommodate computational constraints, we adjust the configuration by training for 90 epochs using the AdamW optimizer with a reduced global batch size of 1024. All other hyperparameters remain unchanged.
    \item \textbf{Strong recipe:} We adopt the improved training recipe and hyperparameter settings proposed by~\citet{beyer2022better} without mixup augmentation, which serves as a stronger baseline.
\end{itemize}

\paragraph{\shortterm regularization.}
We apply the same \shortterm regularization configuration across both training settings.
The regularization setup shares similarities with our CIFAR-10 experiments but includes adaptations for the larger output space.
We employ the label-anchored \shortterm strategy in conjunction with randomized cosine sampling.
To efficiently scale to the 1,000-class output space, we apply PCA to reduce the model outputs to 3 principal components before fitting the polynomials, consistent with the strategy described earlier.
We configure the hyperparameters as follows: regularization strength $\lambda=3$, sampling resolution $r=4$, maximum polynomial degree $K=3$, and number of sampled pairs per minibatch $n_p=512$.
Similar to the CIFAR-10 settings, we apply a sinusoidal ramp-up schedule for $\lambda$ during the first 30 epochs (approx. 1/3 of the training duration) to allow the model to learn adequate representations before enforcing stronger complexity control.

\subsection{Settings for CLIP Fine-Tuning on ImageNet}
\label{sec:imagenet_finetune_settings}

We adhere to the end-to-end fine-tuning protocol outlined in~\citet{wortsman2022robust}.
Specifically, we train the pre-trained models for 10 epochs with a batch size of 512 and a fixed learning rate of $3 \times 10^{-5}$.

\paragraph{\shortterm regularization.}
We adopt the same label-anchored \shortterm strategy and PCA projection as in our ImageNet experiments from scratch.
For this fine-tuning setting, we adjust the hyperparameters to $\lambda=2$ and $n_p=256$ (maintaining $r=4$ and $K=3$).
Notably, we omit the warmup schedule for $\lambda$ to enforce immediate regularization, as required by the short 10-epoch training duration.





\subsection{Settings for Natural Language Processing Tasks}
\label{sec:nlp_settings}

We fine-tune the BERT-base model on the RTE, MRPC, and CoLA datasets.
To ensure a rigorous comparison, we first establish strong baselines by performing a grid search over the following hyperparameters: learning rate $\in \{2\times10^{-5}, 5\times10^{-5}, 1\times10^{-4}\}$, batch size $\in \{16, 32, 64\}$, and training epochs $\in \{5, 10, 20\}$.
All reported results are averaged over 9 random seeds.

\paragraph{Adaptation for embedding interpolation.}
Unlike images, raw text data consists of discrete tokens, prohibiting direct linear interpolation in the input space.
To enable the construction of interpolation paths for \shortterm regularization, we adopt the \textit{embedding interpolation} strategy inspired by~\cite{guo2019augmenting}.
Specifically, let $E(\cdot)$ denote the embedding function that maps input tokens to their vector representations (aggregating word embeddings, positional encodings and segment embeddings).
Given two sentence inputs $x_1, x_2$, we perform linear interpolation on their embedding representations immediately after encoding: $h(\alpha) = \alpha E(x_1) + (1-\alpha) E(x_2)$.
This intermediate manifold is then propagated through the transformer layers to compute the \longterm of the decision trace.  

\paragraph{Method-specific configurations.}
For {mixup}, we employ an embedding interpolation strategy  $\lambda \sim \mathrm{Beta}(\alpha,\alpha)$ with $\alpha=1.0$.
For {\shortterm regularization}, we adopt the label-anchored \shortterm strategy with randomized cosine sampling.
Across all tasks, we fix the sampling resolution to $r=4$ and the maximum polynomial degree to $K=3$.
Consistent with the short fine-tuning duration, we do not apply a warmup schedule for regularization strength $\lambda$.

\medskip
Based on the grid search, the optimal baseline configurations and the task-specific \shortterm hyperparameters (regularization strength $\lambda$ and sampled pairs $n_p$) are set as follows:
\begin{itemize}
	\item \textbf{RTE:} We use a learning rate of $5\times10^{-5}$, batch size 32, and train for 10 epochs. For \shortterm, we set $\lambda=0.5$.
	\item \textbf{MRPC:} We use a learning rate of $5\times10^{-5}$, batch size 16, and train for 10 epochs. For \shortterm, we set $\lambda=0.5$.
	\item \textbf{CoLA:} We use a learning rate of $2\times10^{-5}$, batch size 64, and train for 20 epochs. For \shortterm, we set $\lambda=0.3$.
\end{itemize}

\subsection{Settings for Reinforcement Learning}
\label{sec:rl_settings}

Following standard protocols~\cite{wang_improving_2020}, we train agents on a fixed set of 500 levels (train) and evaluate them on a distinct set of unlimited unseen levels (test).
We adhere to the implementation details and base hyperparameters provided in~\citet{huang2022cleanrl}.
Agents are trained for 50M timesteps, with the exception of \textit{FruitBot}, which is trained for 100M timesteps to accommodate its longer startup phase.

\paragraph{Adaptation for actor regularization.}
We apply our method to the Proximal Policy Optimization (PPO) algorithm, utilizing a CNN-based actor-critic architecture.
The actor network $\pi_\theta(a|s)$ determines the agent's behavior, while the critic $V_\phi(s)$ estimates the value function.
We introduce the \shortterm penalty exclusively to the actor network.
The rationale for this design choice is twofold:
First, the Actor defines the decision boundary; ensuring the smoothness of the policy with respect to state variations directly translates to more robust and generalized action selection.
Second, the value function often requires sharp transitions to accurately reflect sudden changes in expected returns (e.g., cliffs or immediate rewards), making complexity penalization potentially detrimental to value estimation.

\paragraph{\shortterm regularization.}
For the regularization setup, we employ randomized cosine sampling in conjunction with PCA dimensionality reduction.
Specifically, we project the output vectors of the policy network onto a subspace of $m=3$ dimensions.
We adopt the same base hyperparameters for all four tasks: regularization strength $\lambda=0.01$, sampling resolution $r=4$, maximum polynomial degree $K=3$ and sampled pairs $n_p = 2048$ (half of the minibatch size).
    For \textit{Fruitbot}, \textit{Jumper}, and \textit{StarPilot}, we use a sinusoidal ramp-up schedule for the first $1/6$ of the training steps.

Distinct from the supervised learning settings in previous sections, fixed ground-truth action labels are unavailable in reinforcement learning. We thus employ ED regularization without the label-anchoring strategy and directly constrain the \longterm of the policy network to induce a simpler actor.

\section{Computational Efficiency Analysis}
\label{sec:efficiency_analysis}

In this section, we analyze the computational overhead incurred by the additional sampling process.
Our analysis is based on the efficiency-oriented configuration described in \cref{sec:cifar10_settings}, which was shown to yield competitive results in our experiments. Specifically, this setting employs a sampling resolution of $r=4$ and a pair count of $n_p = B/2$, where $B$ denotes the batch size.

Under the label-anchored \shortterm strategy, the two endpoints of the interpolation trajectory are fixed to the ground-truth anchors and do not require model forward passes.
Consequently, gradient propagation is only required for the $r-2=2$ intermediate nodes per sampled pair.
Given that we sample $n_p = B/2$ pairs per iteration, the total number of additional forward passes is:
\[
\text{Additional Passes} = n_p \times (r-2) = \frac{B}{2} \times 2 = B.
\]
This overhead is exactly equivalent to the computational cost of processing one nominal batch. Since the polynomial fitting operations are fully parallelized and incur negligible cost compared to the model forward passes, this implies a theoretical cost multiplier of approximately $2\times$.



To validate this empirically, we conducted time-cost experiments on CIFAR-10 with a batch size of $B=256$. The average time per epoch was $6.14\text{s}$ without regularization, whereas introducing \shortterm regularization increased this to $9.75\text{s}$. We further verified these observations using the CLIP fine-tuning setup ($B=512$), where the average time per iteration increased from $0.44\text{s}$ (without regularization) to approximately $0.90\text{s}$ (with \shortterm). These results align with the theoretical prediction of a $2\times$ increase, as the specific GPU resource demands differ between these two tasks.

\section{Failure Analysis}
\label{app:failure_analysis}
We include a failure case to clarify a limitation of ED regularization. 
Motivated by prior work on simplicity bias, 
we hypothesize that ED regularization may fail when the simpler feature is also the less robust one. 
To test this hypothesis, we construct \textsc{MNIST--CIFAR} following~\citet{shah_pitfalls_2020}, 
a synthetic dataset formed by concatenating MNIST images from classes $0$ and $1$ with CIFAR-10 images from the automobile and truck classes. 
The resulting binary classification task can be solved using either simple MNIST features or more complex CIFAR-10 features.

We train ViT-Tiny with standard training and with ED regularization. 
To assess which component the model relies on, we evaluate standard test accuracy, test accuracy after randomizing the MNIST component, and test accuracy after randomizing the CIFAR-10 component. 
The results are shown in Table~\ref{tab:mnist-cifar-failure}.

\begin{table}[t]
\centering
\caption{
Failure case on \textsc{MNIST--CIFAR}. 
Both the baseline and ED-regularized model rely almost entirely on the simpler MNIST signal. 
Randomizing the MNIST component causes accuracy to collapse, whereas randomizing the CIFAR component has almost no effect.
}
\label{tab:mnist-cifar-failure}
\begin{tabular}{lcc}
\toprule
Metric & Baseline & ED \\
\midrule
Test acc. & $99.85$ & $99.90$ \\
Test acc. (MNIST randomized) & $48.05$ & $47.95$ \\
Test acc. (CIFAR randomized) & $99.85$ & $99.90$ \\
\bottomrule
\end{tabular}
\end{table}

Both models achieve near-perfect standard test accuracy. 
However, when the MNIST component is randomized, accuracy drops to nearly chance level, while randomizing the CIFAR-10 component leaves accuracy essentially unchanged. 
This indicates that both models rely almost entirely on the simpler MNIST signal. 
ED regularization does not reduce this reliance or improve robustness in this setting.

We therefore view \textsc{MNIST--CIFAR} as a plausible failure case of ED regularization: 
when simple features are easier to exploit but less desirable for robust generalization, ED regularization may fail to encourage the use of more generalizable complex features.

\section{Controlled Study on Polynomial Neural Networks}
\label{app:ablation_pnn}

We further examine \longterm in a controlled setting where the notion of algebraic degree is intrinsic to the model class.
Specifically, we consider polynomial neural networks (PNNs), which provide a simple setting for testing whether \longterm can recover the intended complexity ordering of polynomial target functions.
This experiment also complements the discussion in \cref{subsec:connetcing_ed_gen_bounds} by connecting \longterm to settings where complexity and generalization are more directly characterized.

We consider learning polynomial mappings from $\mathbf{x}=(x_1,x_2,x_3)\in\mathbb{R}^3$ to $\mathbf{y}=(y_1,y_2,y_3)\in\mathbb{R}^3$.
We construct three target mappings with increasing algebraic degree:
\begin{align*}
\text{Task 1:}\quad 
\mathbf{y} &= (x_3+2,\; x_2+3,\; x_1+1), \\
\text{Task 2:}\quad 
\mathbf{y} &= (x_1x_2,\; x_2x_3,\; x_1x_3), \\
\text{Task 3:}\quad 
\mathbf{y} &= \left(x_1x_2x_3,\; x_1^2x_2x_3^2,\; x_1^2x_2x_3+x_2^2x_3\right).
\end{align*}
Tasks 4--6 are constructed by multiplying the outputs of Tasks 1--3 by a factor of $2$, respectively, while preserving their algebraic degrees.
This allows us to test the scale dependence of \longterm and its normalized variant.

We train a 4-layer PNN with square activation functions, consisting of three hidden layers followed by a linear output layer.
The model is trained to fit each target mapping nearly perfectly.
We then evaluate \longterm and several variants, including normalized \longterm, Legendre-basis \longterm, and PCA-based output reductions.
The results are summarized in \cref{tab:pnn_controlled_study}.

\begin{table}[t]
	\centering
	\small
	\setlength{\tabcolsep}{5pt}
	\caption{
	Controlled study on polynomial neural networks.
	We report \shortterm and its variants on polynomial target mappings with increasing algebraic degree.
	Tasks 4--6 are scaled versions of Tasks 1--3, respectively.
	}
	\label{tab:pnn_controlled_study}
	\begin{tabular}{lcccccc}
		\toprule
		Task 
		& Algebraic degree
		& \shortterm{} (Chebyshev)
		& \shortterm{}$_{\mathrm{norm}}$ (Chebyshev)
		& \shortterm{} (Legendre)
		& \shortterm{} (PCA-1)
		& \shortterm{} (PCA-2) \\
		\midrule
		Task 1 & $1$ & $0.72$ & $0.29$ & $0.73$ & $1.40$ & $0.70$ \\
		Task 2 & $2$ & $1.17$ & $0.78$ & $1.33$ & $2.12$ & $1.34$ \\
		Task 3 & $5$ & $2.92$ & $1.17$ & $3.64$ & $5.32$ & $3.40$ \\
		Task 4 & $1$ & $1.34$ & $0.28$ & $1.35$ & $2.67$ & $1.34$ \\
		Task 5 & $2$ & $2.33$ & $0.82$ & $2.63$ & $4.13$ & $2.62$ \\
		Task 6 & $5$ & $5.30$ & $1.17$ & $6.64$ & $9.62$ & $6.16$ \\
		\bottomrule
	\end{tabular}
\end{table}

These controlled polynomial-learning experiments reveal several useful properties of \longterm.
First, \longterm preserves the degree ordering of the target mappings:
Tasks 1, 2, and 3 have increasing algebraic degree, and their corresponding \longterm values increase accordingly.
The same ordering also holds for the scaled versions, Tasks 4, 5, and 6.
Second, the standard \longterm is not scale-invariant: scaling the outputs by a factor of $2$ approximately increases the corresponding \longterm values.
In contrast, \shortterm{}$_{\mathrm{norm}}$ is nearly invariant to this scaling.
Third, the same degree ordering is preserved when replacing the Chebyshev basis with the Legendre basis, suggesting that the method does not rely on a specific orthogonal polynomial basis.
Finally, PCA-based output reduction to one or two dimensions still preserves the same ordering, indicating that \longterm remains effective under moderate dimensionality reduction.

\end{document}